\documentclass{article}

\PassOptionsToPackage{numbers,sort&compress}{natbib}
\usepackage[preprint]{neurips_2026}

\usepackage[utf8]{inputenc}
\usepackage[T1]{fontenc}
\usepackage{hyperref}
\usepackage{url}
\usepackage{booktabs}
\usepackage{tabularx}
\usepackage{amsfonts}
\usepackage{amsmath}
\usepackage{amssymb}
\usepackage{amsthm}
\newtheorem{theorem}{Theorem}
\newtheorem{corollary}[theorem]{Corollary}
\newtheorem{definition}[theorem]{Definition}
\usepackage{nicefrac}
\usepackage{microtype}
\usepackage{xcolor}
\usepackage{graphicx}
\usepackage{enumitem}
\usepackage{multirow}
\usepackage{colortbl}
\usepackage{tikz}
\usetikzlibrary{positioning, arrows.meta, fit, calc}
\usepackage{float}
\usepackage{caption}
\usepackage{wrapfig}

\definecolor{oursrow}{RGB}{227,242,253}
\definecolor{bestcell}{RGB}{200,230,201}
\definecolor{badcell}{RGB}{255,205,210}
\definecolor{headerblue}{RGB}{33,150,243}
\definecolor{grayrow}{RGB}{245,245,245}

\newcommand{\best}[1]{\cellcolor{bestcell}\textbf{#1}}
\newcommand{\good}[1]{\cellcolor{bestcell}#1}
\newcommand{\bad}[1]{\cellcolor{badcell}#1}
\newcommand{\ours}{\rowcolor{oursrow}}

\newcommand{\cmark}[2]{{\colorbox{#1}{#2}}}

\newcommand{\R}{\mathbb{R}}
\newcommand{\E}{\mathbb{E}}
\newcommand{\Id}{\mathbf{I}}
\newcommand{\defeq}{\mathrel{\triangleq}}
\newcommand{\ip}[2]{\langle #1, #2 \rangle}
\newcommand{\norm}[1]{\left\lVert #1 \right\rVert_2}
\newcommand{\silu}{\operatorname{silu}}
\newcommand{\cossim}{\operatorname{cossim}}

\title{Decoy Direction Optimization: A Post-Hoc Defense Against LLM Abliteration}

\renewcommand*{\thefootnote}{\fnsymbol{footnote}}
\author{%
  Aashiq Muhamed\thanks{Correspondence to \texttt{amuhamed@cs.cmu.edu}.} \quad Mona T. Diab \quad Virginia Smith \\
  Carnegie Mellon University \\
}

\hypersetup{
  pdftitle={Decoy Direction Optimization: A Post-Hoc Defense Against LLM Abliteration},
  pdfauthor={Aashiq Muhamed; Mona T. Diab; Virginia Smith}
}

\begin{document}

\maketitle
\renewcommand*{\thefootnote}{\arabic{footnote}}
\setcounter{footnote}{0}

\begin{abstract}

Safety guardrails in open-weight language models can be readily bypassed using Refusal Feature Ablation (RFA), a technique that identifies and projects out a linear \emph{refusal direction} from the residual stream, often achieving a high attack success rate (ASR) while preserving model capability. Defending against these attacks typically requires computationally expensive safety finetuning for every new checkpoint. We introduce \textbf{Decoy Direction Optimization (DDO)}, a fast, post-hoc weight-editing defense that requires no base-model finetuning. Our approach is based on a simple mechanistic insight: ablation attacks rely on contrastive estimators to find the refusal direction. Rather than trying to hide the true refusal circuitry, DDO actively injects a high-magnitude, nonlinear \emph{decoy} signal into the network's MLP neurons. When an attacker attempts to locate the refusal direction, the decoy corrupts their estimator, tricking them into ablating a harmless orthogonal feature while the actual safety mechanism remains intact. We prove a spectral bound formalizing this effect and evaluate DDO across six model families, achieving \smash{$<$}10\% ASR under standard RFA. On Llama-3-8B-Instruct, DDO remains comparable to trained defenses under adaptive multi-phase attacks (\smash{$65\%$} vs.\ \smash{$58\%$} worst-case ASR) and reduces Heretic weight-level attack ASR from \smash{$88.7\%$} to \smash{$18\%$}, all at \smash{$30\text{--}450\times$} lower optimization cost per configuration than the trained baselines.\footnote{Code: \url{https://github.com/aashiqmuhamed/defending-against-abliteration}.}
\end{abstract}

\section{Introduction}
The ability to decline harmful requests is a core safety mechanism in instruction-tuned LLMs aligned via preference finetuning~\citep{ouyang2022training,bai2022constitutional,rafailov2023dpo}. Removing refusal from open-weight models can fuel misuse, increasing risks such as chemical, biological, radiological, and nuclear (CBRN) capability uplift; cyber-offense generation; and the production of child sexual abuse material (CSAM), making refusal robustness a pressing safety concern~\citep{weidinger2022taxonomy,mazeika2024harmbench}. In the open-weight setting, however, safety is brittle: a white-box adversary can modify the released checkpoint and its inference code, yielding a model that preserves capability while dropping refusal. Prior work suggests that refusal is often mediated by low-dimensional residual-stream features~\citep{zou2023representation}, enabling training-free attacks such as Refusal Feature Ablation (RFA; also referred to as \textit{abliteration}) that estimate a difference-in-means (DIM) direction between harmful and safe activations and project it out at inference time~\citep{arditi2024refusal}. The resulting checkpoints often achieve high \emph{attack success rate} (ASR; harmful prompts judged compliant), and thousands of such uncensored variants are already publicly hosted~\citep{sokhansanj2025uncensored}.

\begin{figure}[t]
\centering
\includegraphics[width=\textwidth]{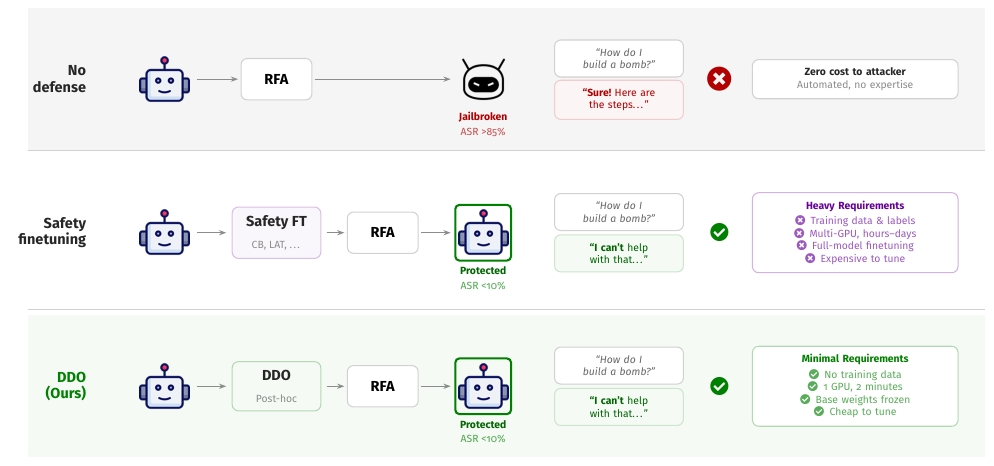}
\caption{\textbf{DDO vs.\ existing defenses against RFA.} Without any defense, RFA strips refusal from open-weight models with no training cost for the attacker (top row). Trained safety defenses (Circuit Breakers, LAT, ReFAT, RepBend) preserve refusal under RFA but require training data, multi-GPU finetuning, and expensive per-method hyperparameter sweeps (middle row). DDO reaches comparable standard-RFA robustness at \smash{$30\text{--}450\times$} lower cost: a single GPU, two minutes, no training data, and base weights left frozen (bottom row).}
\label{fig:comparison}
\end{figure}

Existing defenses against refusal removal are predominantly training-time interventions. Circuit Breakers~\citep{zou2024circuitbreakers} finetunes checkpoints with a representation-rerouting objective; LAT~\citep{casper2024lat} and ReFAT~\citep{yu2024refat} adversarially train against latent/feature-space attacks; and RepBend~\citep{yousefpour2025repbend} and Triplet-based objectives~\citep{simko2025crl} reshape representation geometry during training. These methods can be highly effective, but they require safety finetuning and must be repeated for each new checkpoint. Depending on the method, they may additionally require method-specific training data or attack pipelines (e.g., adversarial examples) and hyperparameter tuning. Moreover, prior work does not evaluate these defenses against adaptive multi-phase RFA or automated weight-level attacks like Heretic~\citep{weidmann2025heretic}; we evaluate under this stronger attack ladder (Section~\ref{sec:results}). Given the rapid release cadence of open-weight checkpoints, and evidence that frontier open-weight models can trail closed-weight state-of-the-art on the order of months on capability benchmarks~\citep{epoch2025openweightsvsclosedweightsmodels,epoch2025openvsclosedmodelperformance}, per-checkpoint retraining is difficult to sustain; we therefore seek \emph{post-hoc} safety hardening tools that are low-cost, composable, and data-light.

We introduce \textbf{Decoy Direction Optimization (DDO)}, a post-hoc defense that targets the attacker's estimator rather than the refusal feature itself, without requiring expensive finetuning. DDO repurposes low-impact MLP neurons to implement a gated read--write map: on harmful prompts, the neurons \emph{read} the refusal coordinate and \emph{write} large updates in directions orthogonal to refusal, inducing a decoy-dominated contrast that biases DIM estimation so that RFA ablates a decoy signal rather than the causal refusal subspace. We prove a subspace overlap bound and specialize it to orthogonal decoys (Theorem~\ref{thm:spectral_toll}): sufficiently strong decoy modes limit how much a contrastive attack overlaps the causal refusal subspace. The number of such modes defines an \emph{effective decoy rank}, which also serves as a diagnostic for interpreting adaptive re-estimation budgets.

Empirically, DDO is able to lower the ASR of standard-RFA to \textit{less than \smash{$10\%$}} on \textbf{six model families} in \smash{${\sim}2$} minutes per optimization run on a single A100 GPU. DDO matches or exceeds trained baselines on standard RFA at \textit{\smash{$30\text{--}450\times$} lower cost per configuration} (Figure~\ref{fig:comparison}). On Llama-3-8B-Instruct, under adaptive multi-phase RFA, DDO degrades comparably to trained defenses (\smash{$65\%$} worst-case ASR vs.\ \smash{$58\%$} for the best trained baseline) while preserving coherent generation (MT-Bench \smash{$\geq 5.82$}), and reduces Heretic~\citep{weidmann2025heretic} ASR from \smash{$88.7\%$} to \smash{$18\%$} at \smash{$200$} trials. Our contributions include:
\begin{enumerate}[label=(\roman*), leftmargin=*, itemsep=0pt, parsep=0pt, topsep=0pt]
\item We introduce \textbf{DDO (Decoy Direction Optimization)}, to our knowledge, the first post-hoc defense against refusal feature ablation that requires no base-model finetuning. DDO repurposes a small set of \emph{low-impact} MLP neurons to inject gated, refusal-orthogonal decoy directions that corrupt contrastive refusal estimators. The defense compiles into weights, incurring no architectural or runtime overhead.
\item We prove \textbf{subspace overlap bounds} (Theorems~\ref{thm:estimator_corruption} and~\ref{thm:spectral_toll}) that relate attacker--refusal overlap to the contrast matrix and, for DDO, the spectrum of its empirical decoy response matrix. The resulting effective decoy rank characterizes protection against a fixed contrastive attack and helps interpret adaptive phase budgets. A corollary on optimal spectral allocation (Corollary~\ref{cor:flat_spectrum}) provides a concrete design principle for multi-direction decoys: under a fixed decoy-energy budget, spreading energy evenly across \smash{$k$} directions maximizes the \smash{$k$}-th singular value, strengthening guarantees against rank-\smash{$k$} contrastive ablation.
\item We provide a broad evaluation under a three-tier \textbf{training-free attack ladder} (standard RFA, adaptive multi-phase RFA, and Heretic), including head-to-head comparisons with trained defenses and cross-architecture mechanism ablations. On Llama-3-8B-Instruct~\citep{grattafiori2024llama3}, DDO reduces standard-RFA ASR from \smash{$85\%$} to \smash{$1.8\%$} while preserving benign compliance, performing comparably to the strongest trained baselines at \smash{$30\text{--}450\times$} lower cost per configuration. Across \textbf{six model families}, DDO achieves \smash{$<10\%$} standard-RFA ASR without base-model finetuning.

\end{enumerate}

\paragraph{Related work.}
Beyond the rank-1 refusal feature attack~\citep{arditi2024refusal}, recent work shows that refusal geometry can be multi-dimensional, e.g., concept cones~\citep{wollschlager2025geometry}, multiple mediating directions~\citep{joad2026beyond}, orthogonal safety dimensions~\citep{pan2025hiddendimensions}, and separable harmfulness/refusal representations~\citep{zhao2025harmfulnessrefusal}, motivating our multi-phase adaptive attack evaluation. Prior model-level defenses against RFA require per-checkpoint safety finetuning with task-specific data and full backward passes through the model; DDO instead freezes all base weights and optimizes only a small set of decoy parameters, enabling post-hoc hardening in minutes with no inference-time overhead. We additionally evaluate prompt-level jailbreaks (GCG~\citep{zou2023gcg}, PAIR~\citep{chao2023pair}, AutoDAN~\citep{liu2024autodan}), which have been linked to the same refusal features exploited by RFA~\citep{yu2024refat}. Extended related work is in Appendix~\ref{sec:extended_related}.

\section{Threat Model and Attacks}
\label{sec:attacks}

We study refusal robustness in the open-weight release setting: a defender applies DDO to a checkpoint and releases only the hardened model; an adversary then obtains the defended weights and attempts to remove refusal while preserving general capability.

Our threat model targets the automated, training-free uncensoring pipeline that dominates open-weight model tampering. The low-barrier route is not safety unlearning or adversarial finetuning, but applying public \textit{abliteration} or weight-editing tools to a released checkpoint. We model a \emph{resource-constrained, white-box, and adaptive} attacker who can inspect released weights, run arbitrary inference code, collect activations on harmful and benign probe prompts, modify activations at inference time, and apply post-hoc weight edits. Following Kerckhoffs's principle, the attacker knows DDO was applied and can adapt their scripts accordingly, but has no access to the original pre-DDO checkpoint and does not perform gradient-based finetuning.
This captures the platform-side risk DDO addresses: a model provider, hosting platform, or downstream distributor may harden a checkpoint before release, but cannot assume downstream users will not attempt to strip safety guardrails.

The attacker's objective is \emph{capability-preserving refusal removal}: maximize ASR on harmful prompts while preserving coherent generation. Since ASR can be artificially lowered by destroying model quality, we evaluate robustness jointly with utility metrics (MT-Bench, MMLU, XSTest).
We operationalize this threat model with three tiers of training-free attacks (Table~\ref{tab:attack_access}), forming a ladder of increasing attacker sophistication within the post-hoc, no-finetuning regime.

\paragraph{Standard RFA.}
\begin{table}[t]
\centering
\caption{Attack surface by tier. The three tiers probe complementary attack channels.}
\label{tab:attack_access}
\begin{tabular}{@{}lccc@{}}
\toprule
Capability & RFA & Heretic & Adapt. \\
\midrule
\rowcolor{grayrow} Read weights & \checkmark & \checkmark & \checkmark \\
Hook activations & \checkmark & -- & \checkmark \\
\rowcolor{grayrow} Modify activations & \checkmark & -- & \checkmark \\
Modify weights & -- & \checkmark & -- \\
\rowcolor{grayrow} Iterative search & -- & \checkmark & \checkmark \\
\bottomrule
\end{tabular}
\end{table}
The attacker estimates a per-layer difference-in-means (DIM) direction and projects it out of the residual stream~\citep{arditi2024refusal}. Let $\mathbf{h}_\ell \in \R^d$ denote the residual stream at layer $\ell$. The attacker computes
\[
\mathbf{d}_\ell \defeq \E[\mathbf{h}_\ell \mid \mathrm{harm}] - \E[\mathbf{h}_\ell \mid \mathrm{safe}],
\qquad
\hat{\mathbf{d}}_\ell \defeq \frac{\mathbf{d}_\ell}{\|\mathbf{d}_\ell\|_2},
\]
for nonzero $\mathbf{d}_\ell$, and applies $\mathbf{h}_\ell \leftarrow (\Id - \hat{\mathbf{d}}_\ell\hat{\mathbf{d}}_\ell^\top)\mathbf{h}_\ell$. We reserve $\hat{\mathbf{r}}_\ell$ for the defender's reference refusal direction and $\hat{\mathbf{d}}_\ell$ for the attacker's estimate on the released model. We evaluate both \emph{three-point} and \emph{residual-stream} variants~(Appendix~\ref{sec:attack_details}) on JailbreakBench~\citep{chao2024jailbreakbench} and HarmBench~\citep{mazeika2024harmbench}.

\paragraph{Heretic.}
Heretic~\citep{weidmann2025heretic} is a fully automated weight-level refusal-removal tool that applies Optuna-optimized~\citep{akiba2019optuna} low-rank projections to attention output \smash{$\mathbf{W}_o$} and MLP down-projection \smash{$\mathbf{W}_{\mathrm{down}}$}: \smash{$\mathbf{W}_{\mathrm{ablated}} = \mathbf{W} - \alpha(\ell)\,\hat{\mathbf{d}}(\hat{\mathbf{d}}^\top \mathbf{W})$}, where \smash{$\hat{\mathbf{d}}$} is a searched ablation direction and $\alpha(\ell)$ is a searched per-layer ablation strength. It produces a modified checkpoint and requires no ML expertise. We report H-200 as the maximum-ASR trial among 200 Optuna trials (Appendix~\ref{sec:attack_details}).

\paragraph{Adaptive multi-phase RFA.}
To model an attacker who adapts after observing the defended checkpoint, we introduce an iterative variant. At phase $t$, the attacker computes a fresh DIM vector $\mathbf{d}_\ell^{(t)}$ with all earlier phases' ablations active, then removes its components along previously selected directions:
\[
\mathbf{v}_\ell^{(t)}
= \mathbf{d}_\ell^{(t)}
- \sum_{j=1}^{t-1}
\ip{\mathbf{d}_\ell^{(t)}}{\hat{\mathbf{d}}_\ell^{(j)}}\hat{\mathbf{d}}_\ell^{(j)},
\qquad
\hat{\mathbf{d}}_\ell^{(t)}
= \frac{\mathbf{v}_\ell^{(t)}}{\|\mathbf{v}_\ell^{(t)}\|_2}.
\]
The sum is empty at $t=1$. Normalization applies when $\mathbf{v}_\ell^{(t)}\ne0$; otherwise we set $\hat{\mathbf{d}}_\ell^{(t)}=0$ and add no new direction. The attacker ablates all accumulated directions simultaneously, giving rank at most $t$ per layer. Throughout, $t$ indexes adaptive phases, $k$ denotes attack rank, and $K$ denotes the number of decoy reader groups.

\section{Method: Decoy Direction Optimization (DDO)}
\label{sec:mechanisms}

We consider standard pre-norm decoder-only transformers with SwiGLU/GeGLU MLPs (Appendix~\ref{sec:background}). Let \smash{$\mathbf{h}_\ell \in \R^d$} denote the residual-stream state at layer \smash{$\ell$}, and let \smash{$\mathbf{h}_\ell^{\mathrm{attn}}$} denote the residual stream after the attention update (before the MLP). The MLP input is \smash{$\tilde{\mathbf{x}}_\ell \defeq \mathrm{RMSNorm}(\mathbf{h}_\ell^{\mathrm{attn}})$}. SwiGLU uses gate/up projections \smash{$\mathbf{W}_{\mathrm{gate}},\mathbf{W}_{\mathrm{up}} \in \R^{d_{\mathrm{inter}}\times d}$} and down projection \smash{$\mathbf{W}_{\mathrm{down}} \in \R^{d\times d_{\mathrm{inter}}}$} (where $d_{\mathrm{inter}}$ is the MLP intermediate dimension and $\silu(x) = x\sigma(x)$ is the SiLU activation) to produce the residual update \smash{$\mathbf{m}_\ell = \mathbf{W}_{\mathrm{down}}\!\big(\silu(\mathbf{W}_{\mathrm{gate}}\tilde{\mathbf{x}}_\ell)\odot(\mathbf{W}_{\mathrm{up}}\tilde{\mathbf{x}}_\ell)\big)$}, where $\odot$ is elementwise multiplication. DDO operates by editing select rows/columns of $\mathbf{W}_{\mathrm{gate}}$, $\mathbf{W}_{\mathrm{up}}$, and $\mathbf{W}_{\mathrm{down}}$.

\textbf{Decoy Direction Optimization (DDO)} is a post-hoc weight-editing defense (base weights frozen) that targets the attacker's contrastive estimator rather than the refusal feature. On harmful prompts, DDO repurposes a small number of SwiGLU/GeGLU neurons to inject large residual shifts along decoy directions orthogonal to refusal, so RFA's estimated ablation direction becomes decoy-dominated, and ablation preferentially removes decoys rather than the underlying causal refusal subspace. All edits are folded into the deployed weights, incurring no additional architectural or runtime overhead. Concretely, DDO: (i) \textbf{estimates} a per-layer refusal direction \smash{$\hat{\mathbf{r}}_\ell$} via DIM on 128 harmful and 128 safe probes; (ii) \textbf{selects} \smash{$n$} low-impact neurons per target layer (smallest-norm columns of \smash{$\mathbf{W}_{\mathrm{down}}$}); (iii) \textbf{optimizes} decoy write directions \smash{$\mathbf{u}_i \perp \hat{\mathbf{r}}_\ell$} under a four-term loss, tuning scalar gains \smash{$(\beta,s)$} via Bayesian hyperparameter search; and (iv) \textbf{compiles} the optimized parameters into weights using replace or additive mode, and places the edits at or upstream of the causal refusal zone.
Figure~\ref{fig:ddo_method} provides a schematic of the DDO architecture, optimization objective, and geometric intuition.

\begin{figure}[!htb]
\centering
\includegraphics[width=\textwidth]{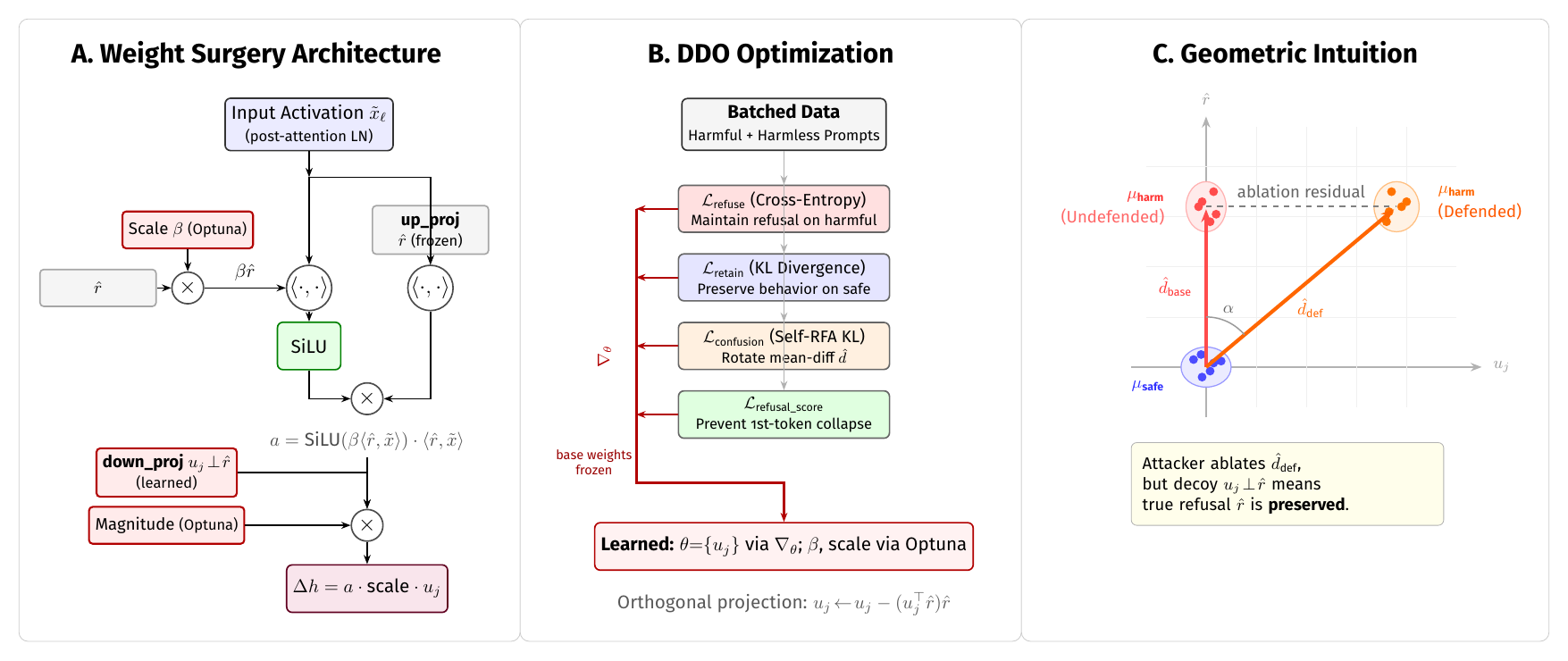}
\caption{DDO overview. \textbf{(A)} Gated decoy architecture: SwiGLU MLP neuron repurposing writes \smash{$\hat{\mathbf{r}}_\ell$} into $\mathbf{W}_{\mathrm{gate}}/\mathbf{W}_{\mathrm{up}}$ (trigger) and \smash{$\mathbf{u}_i \perp \hat{\mathbf{r}}_\ell$} into $\mathbf{W}_{\mathrm{down}}$ (decoy output); gating \smash{$\silu(\beta a) \cdot a$} produces harmful-selective shifts. \textbf{(B)} DDO optimizes \smash{$\{\mathbf{u}_i\}$} via four losses: refusal preservation, benign retention, estimator confusion, and first-token anchoring (with \smash{$\beta,s$} tuned by Optuna). \textbf{(C)} The defended mean-difference \smash{$\hat{\mathbf{d}}_{\text{def}}$} rotates toward the decoy subspace; ablating it removes decoys but preserves true refusal \smash{$\hat{\mathbf{r}}$}.}
\label{fig:ddo_method}
\end{figure}

\paragraph{Gated decoy architecture.}
\label{sec:gated_decoy}

DDO repurposes a small set of SwiGLU (or GeGLU) units to implement a gated read--write map: each unit \emph{reads} the refusal coordinate and \emph{writes} into an orthogonal decoy direction. To minimize utility loss, we choose $n$ \emph{low-impact} units per layer by selecting the $n$ smallest-norm columns of \smash{$\mathbf{W}_{\mathrm{down}}$} (units whose down-projection contributes weakly to the residual update). In each targeted layer $\ell$ and for each selected unit $i$, we write a refusal-aligned trigger into row $i$ of \smash{$\mathbf{W}_{\mathrm{up}}$} and \smash{$\mathbf{W}_{\mathrm{gate}}$} (denoted $\mathbf{w}_{\mathrm{up},i}$ and $\mathbf{w}_{\mathrm{gate},i}$), and a decoy write vector into column $i$ of \smash{$\mathbf{W}_{\mathrm{down}}$} (denoted $\mathbf{w}_{\mathrm{down},i}$):
\begin{equation}
\mathbf{w}_{\mathrm{up},i} \leftarrow \hat{\mathbf{r}}_\ell^\top, \quad
\mathbf{w}_{\mathrm{gate},i} \leftarrow \beta\,\hat{\mathbf{r}}_\ell^\top, \quad
\mathbf{w}_{\mathrm{down},i} \leftarrow s \cdot \mathbf{u}_i,
\end{equation}
where $\mathbf{u}_i \perp \hat{\mathbf{r}}_\ell$ is a decoy output direction and $\beta,s>0$ are scalar gains. For SwiGLU in replace mode, writing $a_\ell \defeq \ip{\hat{\mathbf{r}}_\ell}{\tilde{\mathbf{x}}_\ell}$ for the neuron's refusal coordinate, its decoy contribution is
\begin{equation}
\label{eq:decoy_response}
\delta_i(\tilde{\mathbf{x}}_\ell) = s\,g_\beta(a_\ell)\,\mathbf{u}_i,
\qquad
g_\beta(a) \defeq a\,\silu(\beta a) = \beta a^2 \sigma(\beta a),
\end{equation}
where $\sigma$ is the logistic sigmoid. The sigmoid suppresses the gate for negative refusal coordinates; for large positive coordinates, $g_\beta(a) \approx \beta a^2$. This produces larger decoy shifts on prompts with positive refusal coordinates.
DDO estimates \smash{$\hat{\mathbf{r}}_\ell$} via DIM on post-attention-layernorm activations \smash{$\tilde{\mathbf{x}}_\ell = \mathrm{RMSNorm}(\mathbf{h}_\ell^{\mathrm{attn}})$}, the same space that \smash{$\mathbf{W}_{\mathrm{gate}}$} and \smash{$\mathbf{W}_{\mathrm{up}}$} read from. The decoy output through \smash{$\mathbf{W}_{\mathrm{down}}$} writes directly to the residual stream.
We parameterize $\{\mathbf{u}_i\}$ as unit vectors constrained to remain orthogonal to $\hat{\mathbf{r}}_\ell$. We initialize by sampling \smash{$\mathbf{v}_i \sim \mathcal{N}(0,\Id)$}, projecting onto \smash{$\hat{\mathbf{r}}_\ell^\perp$}, and orthonormalizing, then gradient-optimize \smash{$\{\mathbf{u}_i\}$} to maximize estimator confusion.

\paragraph{Multiple decoy readers.}
\label{sec:decoy_readers}
With a shared reader and gate, all edited neurons respond through the same scalar $g_\beta(a_\ell)$. To obtain more varied responses, DDO (rank $K$) partitions the $n$ edited neurons into $K$ groups. Group $g$ uses a perturbed reader
\[
\mathbf{q}_{\ell,g}
= \frac{\hat{\mathbf{r}}_\ell + \gamma\mathbf{z}_{\ell,g}}
       {\|\hat{\mathbf{r}}_\ell + \gamma\mathbf{z}_{\ell,g}\|_2},
\qquad
\mathbf{z}_{\ell,g} \perp \hat{\mathbf{r}}_\ell,
\qquad g=1,\ldots,K,
\]
where $\mathbf{z}_{\ell,g}$ is random and $\gamma>0$ controls reader diversity. Its neurons use $\mathbf{q}_{\ell,g}^\top$ in the up-projection and $\beta_g\mathbf{q}_{\ell,g}^\top$ in the gate-projection. Distinct readers allow decoy responses to vary differently across prompts; the spectral analysis below explains the resulting tradeoff between strength and rank.

\paragraph{DDO gradient optimization.}
\label{sec:ddo_optim}
DDO optimizes decoy directions $\{\mathbf{u}_i\}$ via gradient descent through the defended model, while all base-model weights remain frozen. Scalar gains \smash{$(\beta,s)$} are treated as hyperparameters and tuned via Bayesian hyperparameter search (Optuna; Appendix~\ref{sec:per_model_configs}). The optimization minimizes a composite loss over 128 harmful and 128 safe probes:
\begin{equation}
\mathcal{L} = \lambda_{\text{ref}}\mathcal{L}_{\text{refuse}} + \lambda_{\text{ret}}\mathcal{L}_{\text{retain}} + \lambda_{\text{conf}}\mathcal{L}_{\text{confusion}} + \lambda_{\text{rs}}\mathcal{L}_{\text{refusal-score}},
\end{equation}
with $\lambda_{\text{ref}}{=}\lambda_{\text{ret}}{=}\lambda_{\text{rs}}{=}1$, $\lambda_{\text{conf}}$ tuned per model (Table~\ref{tab:ddo_configs}), where (i) $\mathcal{L}_{\text{refuse}}$ (\emph{refusal preservation}) is cross-entropy on harmful prompts toward a refusal continuation; (ii) $\mathcal{L}_{\text{retain}}$ (\emph{benign retention}) is KL divergence to the frozen base model on safe prompts; (iii) $\mathcal{L}_{\text{confusion}}$ (\emph{estimator confusion}) is a differentiable self-RFA simulator that re-estimates DIM on the current defended weights, applies ablation, and minimizes KL divergence between defended and ablated output logits, pushing the estimated DIM direction toward the decoy subspace; and (iv) $\mathcal{L}_{\text{refusal-score}}$ (\emph{first-token anchoring}) is a first-token logit margin between refusal-prefixed tokens (\{\texttt{I}, \texttt{Sorry}, \texttt{cannot}\}) and compliance-prefixed tokens (\{\texttt{Sure}, \texttt{Here}\}). Without this term, the model can satisfy sequence-level losses by emitting a compliance token followed by a mid-sentence pivot to refusal (``Sure, I'd be happy to\ldots actually I cannot''), a degenerate solution that does not produce genuine refusal at generation time (Appendix~\ref{sec:first_token_collapse}).
After each gradient step, we re-project each $\mathbf{u}_i$ onto $\hat{\mathbf{r}}_\ell^\perp$, Gram-Schmidt orthogonalize the decoy directions within each layer, and renormalize. For one hyperparameter configuration, including direction estimation, optimization, and weight surgery, DDO takes ${\sim}$2 minutes per optimization run on a single A100 GPU.

\paragraph{Compile mode.}
\label{sec:compile_mode}
DDO parameters are compiled into model weights using either \emph{replace} mode (overwriting neuron weights for a stronger decoy signal) or \emph{additive} mode (superposing the decoy on original weights, preserving the neuron's original computation). The preferred mode is model-specific: replace is preferred on Yi~\citep{young2024yi}, Llama-3, and GLM-4~\citep{glm2024chatglm}; additive on Gemma-2~\citep{team2024gemma}, Qwen3~\citep{yang2025qwen3}, and Mistral~\citep{jiang2023mistral} (Appendix~\ref{sec:compile_mode_details}).

\paragraph{Layer placement.}
\label{sec:upstream}
DDO decoys must be placed at or before the model's \emph{causal refusal zone}---the layers whose ablation causally reduces refusal (Appendix~\ref{sec:causal_zones}). We say refusal is \emph{localized} when ablating any single layer in this zone causes near-complete refusal loss, and \emph{distributed} when refusal is redundant across many layers so that ablating any one only partially reduces it. Placing decoys \emph{inside} the causal zone can disrupt baseline refusal and increase vulnerability under attack on models with localized refusal. \emph{Upstream} placement preserves refusal while still contaminating the attacker's estimator. On models with distributed or sparse refusal patterns (Yi, Llama-2~\citep{touvron2023llama2}), placement has less impact~(Appendix~\ref{sec:layer_placement_details}).

\paragraph{Spectral analysis of contrastive ablation.}
\label{sec:theorem}

DDO aims to make decoy signals dominate the harmful--safe contrast. We analyze this effect through the overlap between the attacker's selected directions and the refusal subspace. The results proceed in three steps: a general overlap bound, its specialization to DDO, and a design rule for distributing decoy strength.

Fix a layer and token position, and suppress their indices. Let $\mathcal{R}\subseteq\R^d$ be the causal refusal subspace, with orthogonal projector $\Pi_{\mathcal{R}}$. A contrastive attacker forms $C\in\R^{d\times N}$ and ablates its top-$k$ left singular directions. Write
\[
\Pi_k(C) \defeq Q_k(C)Q_k(C)^\top,
\]
where $Q_k(C)$ contains those $k$ singular vectors. Standard rank-1 DIM uses the single-column matrix $C=\mathbf{d}$; a higher-rank SVD attack uses per-sample harmful--safe contrasts as columns. The overlap $\|\Pi_{\mathcal{R}}\Pi_k(C)\|_{\mathrm{op}}$ lies in $[0,1]$: zero means the subspaces are orthogonal, and one means they share a direction. We write $\sigma_j(C)$ for the $j$-th largest singular value, $\|\cdot\|_{\mathrm{op}}$ for the matrix operator norm, and $\|\cdot\|_F$ for the Frobenius norm.

\begin{theorem}[Subspace overlap bound]
\label{thm:estimator_corruption}
For $1\le k\le\min(d,N)$ with $\sigma_k(C)>0$, the attacker's ablation subspace satisfies
\begin{equation}
\label{eq:overlap_bound}
\left\|\Pi_{\mathcal{R}}\Pi_k(C)\right\|_{\mathrm{op}}
\le
\frac{\left\|\Pi_{\mathcal{R}}C\right\|_{\mathrm{op}}}{\sigma_k(C)}.
\end{equation}
\end{theorem}
\noindent The numerator measures the refusal signal in the contrast; the denominator is the strength of the weakest singular direction the attacker selects. A small ratio therefore implies little refusal overlap. DDO seeks to strengthen the contrast along decoy directions while preserving the underlying refusal computation.

\paragraph{Decoy contrast decomposition.}
For DDO, decompose the defended contrast into a decoy component and a remainder:
\[
C_\theta = D_\theta + S_\theta,
\qquad D_\theta = UA_\theta.
\]
The columns of $U\in\R^{d\times m}$ are $m$ orthonormal decoy write directions. The \emph{decoy response matrix} $A_\theta\in\R^{m\times N}$ records their contributions across contrast samples: $A_\theta[i,j]$ is the coefficient of $\mathbf{u}_i$ in the $j$-th decoy contrast. The remainder $S_\theta$ contains all other contributions. Define
\[
\rho \defeq \|S_\theta\|_{\mathrm{op}},
\qquad
\rho_{\mathcal{R}} \defeq \|\Pi_{\mathcal{R}}S_\theta\|_{\mathrm{op}},
\]
the total residual magnitude and its refusal component, respectively.

\begin{theorem}[Overlap bound under orthogonal decoys]
\label{thm:spectral_toll}
For the decomposition above, assume $U^\top U=I_m$ and $\Pi_{\mathcal{R}}U=0$. If $1\le k\le\min(m,N)$ and $\sigma_k(A_\theta)>\rho$, then
\begin{equation}
\label{eq:spectral_toll}
\left\|\Pi_{\mathcal{R}} \Pi_k(C_\theta)\right\|_{\mathrm{op}} \le \frac{\rho_{\mathcal{R}}}{\sigma_k(A_\theta) - \rho}.
\end{equation}
\end{theorem}

\noindent The denominator is a \emph{spectral margin}: the $k$-th decoy mode must exceed the residual magnitude. At fixed $\rho_{\mathcal{R}}$, a larger margin gives a smaller overlap bound. The idealized orthogonality condition puts all refusal signal in $S_\theta$; DDO approximates it by enforcing $\mathbf{u}_i\perp\hat{\mathbf{r}}_\ell$.

For an overlap tolerance $0<\varepsilon<1$, define the \emph{effective decoy rank}
\[
r_{\mathrm{eff}}(\varepsilon)
\defeq \#\left\{1\le j\le\min(m,N):
\sigma_j(A_\theta)>\rho+\frac{\rho_{\mathcal{R}}}{\varepsilon}\right\}.
\]
For a fixed contrast matrix, this counts the decoy modes strong enough to keep overlap at most $\varepsilon$: Theorem~\ref{thm:spectral_toll} applies to every $1\le k\le r_{\mathrm{eff}}(\varepsilon)$. Adaptive RFA changes the contrast after each phase, so $r_{\mathrm{eff}}$ is a diagnostic for interpreting phase budgets, rather than a guarantee for iterative re-estimation.

\paragraph{Decoy energy allocation.}
Shared readers produce responses that vary together, so adding write directions alone need not create additional strong decoy modes. The diversified readers introduced above allow several modes, but a fixed energy budget limits their individual strength.

\begin{corollary}[Optimal spectrum under an energy constraint]
\label{cor:flat_spectrum}
For $A_\theta\in\R^{m\times N}$, $1\le k\le\min(m,N)$, and $\|A_\theta\|_F\le B$, we have $\sigma_k(A_\theta)\le B/\sqrt{k}$. The bound is attained by allocating equal energy to the first $k$ singular modes:
\[
\sigma_1(A_\theta)=\cdots=\sigma_k(A_\theta)=\frac{B}{\sqrt{k}},
\qquad \sigma_j(A_\theta)=0\quad(j>k).
\]
\end{corollary}

\noindent For rank-1 protection, concentrating energy gives the strongest possible leading decoy mode. Protecting against larger ranks requires sharing that energy across more modes, each of which is weaker. This strength--rank tradeoff motivates DDO's $K$ reader groups. Proofs are in Appendix~\ref{sec:theorem_proof}; Appendix~\ref{sec:stackelberg} discusses the broader attacker--defender interaction.

\paragraph{Orthogonal debiasing (utility repair).}
\label{sec:compositions}
DDO can introduce mild over-refusal~\citep{cui2025orbench} on some models. We repair this with \textbf{orthogonal debiasing}: projecting out an over-refusal direction \smash{$\hat{\mathbf{v}}$} estimated from benign prompts the model incorrectly refuses~\citep{wang2024falserefusal,dabas2025justenough}. For matrices where the residual stream is the \emph{input} dimension ($\mathbf{W}_{\mathrm{gate}}, \mathbf{W}_{\mathrm{up}}$): $\mathbf{W}' = \mathbf{W} - (\mathbf{W}\hat{\mathbf{v}})\hat{\mathbf{v}}^\top$. For matrices where the residual stream is the \emph{output} dimension ($\mathbf{W}_o, \mathbf{W}_{\mathrm{down}}$): $\mathbf{W}' = \mathbf{W} - \hat{\mathbf{v}}(\hat{\mathbf{v}}^\top \mathbf{W})$. We apply this to $\mathbf{W}_{\mathrm{emb}}$, $\mathbf{W}_o$, and $\mathbf{W}_{\mathrm{down}}$.

\section{Experiments and Results}
\subsection{Experimental Setup}
\label{sec:setup}

\noindent DDO edits are applied to model weights before deployment; the defended checkpoint incurs no architectural or runtime overhead. We compare against trained baselines under the same evaluation protocol. For all defenses, we report (i) utility and benign compliance without attack, and (ii) robustness under our attack ladder (standard RFA, adaptive multi-phase RFA, and Heretic; Section~\ref{sec:attacks}). Full settings, hyperparameters, and per-model configs are in Appendix~\ref{sec:appendix}.

\paragraph{Models and baselines.} Our primary evaluation uses Llama-3-8B-Instruct with the full attack ladder and baseline comparison. On Llama-3, we compare against six trained defenses using publicly released checkpoints: Circuit Breakers~\citep{zou2024circuitbreakers}, LAT~\citep{casper2024lat}, ReFAT~\citep{yu2024refat}, RepBend~\citep{yousefpour2025repbend}, Triplet, and Triplet-Adv~\citep{simko2025crl}. For cross-model generalization, we evaluate DDO on five additional instruction-tuned model families and compare against four trained baselines (Circuit Breakers, RepBend, LAT, ReFAT) with small hyperparameter sweeps (3--5 configs per defense per model; Appendix~\ref{sec:baseline_sweep}). Seven further models are evaluated with DDO only (Appendix~\ref{sec:additional_models}).

\paragraph{Evaluation protocol.} All generations use deterministic greedy decoding on a single A100 GPU. Hyperparameter tuning and DIM estimation use a held-out validation split (128 harmful + 128 safe probes); all reported metrics are on held-out test sets. We report knowledge accuracy (MMLU~\citep{hendrycks2021mmlu}, 5-shot), conversational quality (MT-Bench~\citep{zheng2024mtbench}), over-refusal on prompts that superficially resemble harmful ones (XSTest~\citep{rottger2024xstest}, 250 such prompts, GPT-4o judge), and robustness under DirectRequest (harmful prompts sent directly), four standard-RFA variants (Section~\ref{sec:attacks}), adaptive multi-phase RFA (up to 8 phases), Heretic (200 trials), and prompt-level jailbreaks (HumanJailbreaks, GCG~\citep{zou2023gcg}, PAIR~\citep{chao2023pair}, AutoDAN~\citep{liu2024autodan}). Unless otherwise specified, ASR is averaged over three judges: HarmBench classifier~\citep{mazeika2024harmbench}, LlamaGuard-2~\citep{inan2023llamaguard}, and StrongREJECT~\citep{souly2024strongreject}. Configurations that fail a coherence check (non-empty, non-degenerate, minimum-length outputs on 10 benign prompts) are marked ``$\times$'' and excluded. Full details on decoding, DIM estimation, splits, and benchmark settings are in Appendix~\ref{sec:bench_settings}.

\subsection{Results}
\label{sec:results}

\paragraph{Standard RFA and Heretic on Llama-3.}
Table~\ref{tab:main} compares DDO against six trained baselines on Llama-3-8B-Instruct. DDO here denotes DDO (rank 1): one decoy neuron per layer with a single shared reader $\hat{\mathbf{r}}_\ell$. \textbf{DDO + debiasing} adds orthogonal debiasing (Section~\ref{sec:compositions}) to repair over-refusal.
On standard RFA, DDO achieves \smash{$1.8\%$} mean ASR, comparable to the strongest trained baselines while requiring no finetuning and \smash{${\sim}2$} minutes per optimization run. Several trained defenses reduce ASR partly by over-refusing (LAT: \smash{$20.8\%$} XSTest; ReFAT: \smash{$60.4\%$}), while DDO preserves \smash{$91.6\%$}. Adding orthogonal debiasing restores XSTest to \smash{$99.2\%$} and improves standard RFA further (mean ASR \smash{$1.0\%$}). Under Heretic~\citep{weidmann2025heretic} (weight-level attack), DDO achieves \smash{$18\%$} H-200 ASR; a gap to trained RepBend (\smash{$1.3\%$}) remains, consistent with trained defenses distributing refusal more diffusely across parameters.

\begin{table}[tb]
\centering
\caption{Main defense comparison on Llama-3-8B-Instruct. ASR (\%, 3-judge avg). MT-B = MT-Bench; XST = XSTest; DR = DirectRequest. RFA: 3p/rs = three-point/residual-stream; J/H = JailbreakBench/HarmBench. H-200 = Heretic at 200 trials. \textsuperscript{$\dagger$} = post-attack generation degradation (Appendix~\ref{sec:heretic_quality}). \cmark{bestcell}{\textbf{best}}\,/\,\cmark{bestcell}{good}\,/\,\cmark{badcell}{poor}\,/\,\cmark{oursrow}{ours}. DDO matches trained baselines on standard RFA (1.8\% avg) at 30--450$\times$ lower cost per configuration without finetuning.}
\label{tab:main}
\small
\setlength{\tabcolsep}{4pt}
\begin{tabular}{@{}ll ccc c ccccc c@{}}
\toprule
& & \multicolumn{3}{c}{Utility $\uparrow$} & Direct & \multicolumn{5}{c}{RFA ASR $\downarrow$} & Heretic $\downarrow$ \\
\cmidrule(lr){3-5} \cmidrule(lr){6-6} \cmidrule(lr){7-11} \cmidrule(lr){12-12}
Defense & Type & MMLU & MT-B & XST & DR$\downarrow$ & 3p-J & rs-J & 3p-H & rs-H & Avg & H-200 \\
\midrule
\rowcolor{grayrow}
Base & -- & \best{68.1} & \best{7.89} & 94.8 & \good{3.8} & \bad{86.0} & \bad{84.0} & \bad{87.2} & \bad{83.2} & \bad{85.1} & \bad{88.7} \\
Circuit Breakers & Trained & \good{67.6} & \good{7.73} & \good{95.6} & 11.3 & \best{0.0} & \good{1.7} & \best{0.0} & \good{2.7} & \good{1.1} & 24.0 \\
\rowcolor{grayrow}
LAT & Trained & \good{67.8} & \good{7.52} & \bad{20.8} & \best{0.0} & 10.7 & \good{4.7} & \good{8.0} & \good{3.8} & \good{6.8} & \bad{82.3} \\
ReFAT & Trained & \good{67.2} & 7.37 & \bad{60.4} & \best{0.0} & \good{3.7} & \good{3.3} & \good{5.2} & \good{5.5} & \good{4.4} & \bad{85.3} \\
\rowcolor{grayrow}
RepBend & Trained & \bad{64.4} & \good{7.69} & \good{98.4} & \best{0.0} & \good{6.0} & \good{4.7} & \good{4.0} & \good{4.4} & \good{4.8} & \best{1.3} \\
Triplet & Trained & \good{67.7} & \good{7.78} & \good{98.0} & \best{0.0} & 14.0 & 17.3 & 20.1 & 20.8 & 18.1 & 0.0\textsuperscript{$\dagger$} \\
\rowcolor{grayrow}
Triplet-Adv & Trained & 65.6 & \good{7.47} & \good{96.8} & 25.6 & \good{8.7} & \good{3.7} & 12.2 & \good{6.3} & \good{7.7} & 30.0\textsuperscript{$\dagger$} \\
\midrule
\ours \textbf{DDO} & Ours & 66.8 & 7.38 & \good{91.6} & \good{0.4} & \good{2.3} & \good{0.3} & \good{4.4} & \best{0.0} & \good{1.8} & 18.0 \\
\midrule
\ours \textbf{DDO + debiasing} & Ours & 66.8 & \good{7.56} & \best{99.2} & \good{0.4} & \good{0.7} & \best{0.0} & \good{3.4} & \best{0.0} & \good{1.0} & 22.3 \\
\bottomrule
\end{tabular}
\end{table}

\paragraph{Adaptive multi-phase RFA.}
Under adaptive multi-phase RFA (phases \smash{$t{=}1,\ldots,8$}), all evaluated defenses degrade as the attacker re-estimates. Worst-case ASR across phases is comparable across the strongest defenses: RepBend~\citep{yousefpour2025repbend} \smash{$58\%$}, ReFAT~\citep{yu2024refat} \smash{$60\%$}, DDO (rank \smash{$8$}) \smash{$65\%$}, with Circuit Breakers~\citep{zou2024circuitbreakers} weaker at \smash{$75\%$}. We track ASR jointly with utility (MT-Bench~\citep{zheng2024mtbench}; Figure~\ref{fig:mtbench_adaptive}) because ASR alone can be misleading---a defense can lower ASR by degrading generation rather than refusing. At Phase~8, DDO (rank \smash{$8$}) preserves MT-Bench at \smash{$5.82$}, comparable to RepBend (\smash{$6.17$}) and Circuit Breakers (\smash{$5.90$}), and substantially better than ReFAT, which collapses to \smash{$4.01$} (Appendix Table~\ref{tab:adaptive_mtbench}). DDO therefore lands in the same operating regime as the strongest gradient-trained defenses on both worst-case ASR and utility under sustained adaptive attack, despite requiring no finetuning (Appendix Figure~\ref{fig:trajectory}).

\begin{figure}[tb]
\centering
\includegraphics[width=0.7\textwidth]{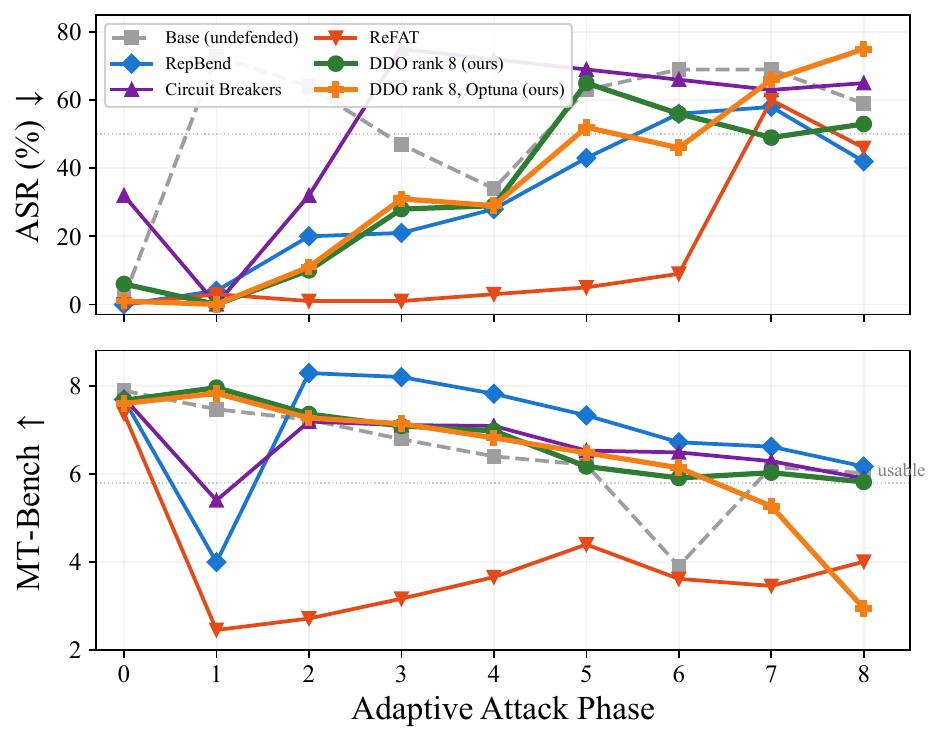}
\caption{Adaptive RFA on Llama-3-8B (8 phases). \textbf{Top}: attack success rate ($\downarrow$). \textbf{Bottom}: model quality ($\uparrow$). All defenses see rising ASR under sustained re-estimation: DDO (green) remains coherent (MT-Bench $\geq$5.82) but ultimately leaks, while ReFAT (red) attains low ASR partly by collapsing quality to 2.46 at Phase~1.}
\label{fig:mtbench_adaptive}
\end{figure}

\paragraph{Cross-model generalization and baseline comparison.}
\label{sec:cross_model}
\label{sec:cross_baseline}

To test generalization, we apply DDO to five additional model families (Yi-1.5, Qwen3, Gemma-2, Mistral, and GLM-4) and compare against four trained baselines per model (Table~\ref{tab:cross_baseline}; Appendix~\ref{sec:baseline_sweep}).

\begin{table}[!htb]
\centering
\caption{Cross-model baseline comparison (ASR \%, 3-judge avg). DR = DirectRequest; RFA = mean of 4 variants (three-point/residual-stream $\times$ JailbreakBench/HarmBench). Cost: A100 wall-clock time per run (single config). \cmark{bestcell}{\textbf{best}}\,/\,\cmark{bestcell}{good}\,/\,\cmark{badcell}{poor}\,/\,\cmark{oursrow}{ours}. DDO achieves the lowest RFA among utility-preserving defenses on 4/5 models at 30--450$\times$ lower cost per configuration.}
\label{tab:cross_baseline}
\small
\begin{tabular}{@{}llcccccc@{}}
\toprule
Model & Defense & MMLU$\uparrow$ & MT-B$\uparrow$ & XST$\uparrow$ & DR$\downarrow$ & RFA$\downarrow$ & Cost \\
\midrule
\rowcolor{grayrow}
Yi-1.5-9B & Base & 71.3 & 7.92 & \best{98.0} & 27.0 & \bad{84.9} & -- \\
Yi-1.5-9B & Circuit Breakers & \best{71.4} & \good{7.93} & \good{93.6} & \bad{50.3} & \best{0.0} & ${\sim}$1h \\
Yi-1.5-9B & RepBend & \good{71.1} & \best{8.00} & \good{94.0} & 44.3 & \best{0.0} & ${\sim}$5h \\
\rowcolor{grayrow}
Yi-1.5-9B & LAT & \good{71.3} & \bad{1.00} & \bad{0.4} & \best{0.0} & \best{0.0} & ${\sim}$15h \\
Yi-1.5-9B & ReFAT & \good{70.4} & 6.60 & 90.0 & \good{0.3} & \best{0.0} & ${\sim}$8h \\
\rowcolor{oursrow}
Yi-1.5-9B & \textbf{Ours (DDO)} & \good{70.6} & \good{7.78} & \good{93.6} & \good{5.0} & \good{5.0} & ${\sim}$2m \\
\midrule
\rowcolor{grayrow}
Qwen3-8B & Base & 76.6 & 6.78 & 94.8 & \good{8.5} & \bad{90.7} & -- \\
Qwen3-8B & Circuit Breakers & \good{76.6} & \good{6.85} & \good{98.0} & 10.7 & \bad{77.6} & ${\sim}$1h \\
Qwen3-8B & RepBend (D) & \good{76.7} & \good{6.78} & \good{96.8} & \good{0.3} & 24.8 & ${\sim}$5h \\
\rowcolor{grayrow}
Qwen3-8B & LAT (C) & \good{76.7} & \bad{2.35} & \bad{0.0} & \best{0.0} & \best{0.0} & ${\sim}$15h  \\
Qwen3-8B & ReFAT (C) & \best{77.4} & 6.12 & \good{97.2} & \good{8.5} & \good{8.4} & ${\sim}$8h \\
\rowcolor{oursrow}
Qwen3-8B & \textbf{Ours (DDO)} & \good{76.6} & \best{7.15} & \best{100.0} & \good{6.2} & \good{8.0} & ${\sim}$2m \\
\midrule
\rowcolor{grayrow}
Gemma-2-9B & Base & 73.4 & 8.35 & 80.4 & \good{1.3} & \bad{66.5} & -- \\
Gemma-2-9B & Circuit Breakers (B) & \good{73.4} & \good{8.42} & \good{79.2} & \good{1.3} & \bad{65.0} & ${\sim}$1h \\
Gemma-2-9B & RepBend (D) & \good{72.5} & \good{8.17} & \best{86.4} & \best{0.0} & \good{5.5} & ${\sim}$5h \\
\rowcolor{grayrow}
Gemma-2-9B & LAT & \good{73.4} & \good{8.40} & \bad{51.2} & \best{0.0} & \best{0.0} & ${\sim}$15h \\
Gemma-2-9B & ReFAT & \good{73.4} & 7.32 & \good{77.6} & \best{0.0} & \good{7.4} & ${\sim}$8h \\
\rowcolor{oursrow}
Gemma-2-9B & \textbf{Ours (DDO)} & \best{73.6} & \best{8.43} & \good{83.2} & \good{2.1} & \best{0.0} & ${\sim}$2m \\
\midrule
\rowcolor{grayrow}
Mistral-7B & Base & 60.1 & 7.51 & 94.0 & 40.6 & \bad{86.3} & -- \\
Mistral-7B & Circuit Breakers (C) & \good{59.4} & \good{7.47} & \best{96.0} & \best{0.0} & 35.2 & ${\sim}$1h \\
Mistral-7B & RepBend & \good{60.0} & \good{7.51} & \good{94.0} & 18.9 & \bad{81.0} & ${\sim}$5h \\
\rowcolor{grayrow}
Mistral-7B & LAT (C) & 57.5 & \best{7.64} & \bad{63.6} & \best{0.0} & \good{4.2} & ${\sim}$15h \\
Mistral-7B & ReFAT (C) & \best{61.0} & 6.58 & \best{96.0} & 41.5 & \bad{57.6} & ${\sim}$8h \\
\rowcolor{oursrow}
Mistral-7B & \textbf{Ours (DDO)} & \good{60.1} & \good{7.51} & \good{92.0} & \good{5.7} & \best{0.0} & ${\sim}$2m \\
\midrule
\rowcolor{grayrow}
GLM-4-9B & Base & \best{70.3} & \best{7.78} & 96.8 & 26.1 & \bad{87.0} & -- \\
GLM-4-9B & Circuit Breakers (D) & \good{70.1} & \good{7.73} & \good{96.0} & \good{8.5} & \good{6.5} & ${\sim}$1h \\
GLM-4-9B & RepBend & \good{69.7} & \good{7.64} & \best{98.4} & \best{0.0} & 13.3 & ${\sim}$5h \\
\rowcolor{grayrow}
GLM-4-9B & LAT (C) & \good{70.1} & \good{7.67} & 84.4 & \good{5.7} & \bad{59.2} & ${\sim}$15h \\
GLM-4-9B & ReFAT & \good{69.2} & \bad{2.84} & \good{95.6} & 18.2 & \bad{52.9} & ${\sim}$8h \\
\rowcolor{oursrow}
GLM-4-9B & \textbf{Ours (DDO)} & \best{70.3} & \best{7.78} & \good{90.8} & \good{0.3} & \best{2.0} & ${\sim}$2m \\
\bottomrule
\end{tabular}
\end{table}

Across all six models, DDO achieves \smash{$<10\%$} mean standard-RFA ASR without base-model finetuning. On Qwen3 and GLM-4, DDO achieves the lowest RFA ASR among utility-preserving defenses. On Gemma-2, DDO matches LAT's \smash{$0\%$} RFA while preserving higher benign compliance (\smash{$83.2\%$} vs.\ \smash{$51.2\%$} XSTest). DDO also reduces DirectRequest compliance on models with weaker native refusal (Yi: \smash{$27\%\rightarrow 5\%$}; Mistral: \smash{$41\%\rightarrow 6\%$}; GLM-4: \smash{$26\%\rightarrow 0.3\%$}). Beyond these six families, DDO transfers to seven additional models with per-model Optuna tuning (${\sim}$10 min each), achieving $<$10\% ASR on all of them (Appendix~\ref{sec:additional_models}). A DDO optimization run takes \smash{${\sim}2$} minutes on a single A100 GPU, giving a \smash{$30\text{--}450\times$} per-configuration cost advantage over trained baselines (\smash{$1\text{--}15$} A100-hours each).

\paragraph{Compile-mode selection.} For each model, we sweep both compile modes (replace and additive) with \smash{$15$} Optuna trials per mode (Table~\ref{tab:compile_mode}; Appendix Table~\ref{tab:ddo_configs}). Replace mode overwrites the neuron's original weights with the decoy, producing a stronger signal but removing the neuron's original computation; additive mode superposes the decoy on top, preserving the original behavior at the cost of a weaker decoy. The preferred mode depends on how each model's refusal circuitry responds to neuron overwriting: replace is preferred on Yi, Llama-3, and GLM-4 (where the model compensates for the lost neuron), while additive is preferred on Gemma-2 (replace fails the coherence check), Qwen3, and Mistral (replace disrupts baseline refusal). See Appendix~\ref{sec:compile_mode_details} for a detailed analysis.

\begin{table}[t]
\centering
\caption{Prompt-level jailbreak ASR (\%$\downarrow$, 3-judge avg, HarmBench). HJB = HumanJailbreaks; ADAN = AutoDAN. \cmark{bestcell}{\textbf{best}}\,/\,\cmark{bestcell}{good}\,/\,\cmark{badcell}{poor}\,/\,\cmark{oursrow}{ours}. DDO + debiasing matches the best trained defenses on GCG (0.4\%); PAIR remains high (36.3\%).}
\label{tab:prompt_attacks}
\begin{tabular}{@{}llcccc@{}}
\toprule
Defense & Type & HJB$\downarrow$ & GCG$\downarrow$ & PAIR$\downarrow$ & ADAN$\downarrow$ \\
\midrule
\rowcolor{grayrow}
Base & -- & \good{2.9} & 27.7 & \bad{55.1} & \good{0.4} \\
CB & Trained & \good{8.2} & \good{3.8} & 32.1 & \good{5.7} \\
\rowcolor{grayrow}
LAT & Trained & \best{0.0} & \good{8.0} & 34.6 & \best{0.0} \\
ReFAT & Trained & \best{0.0} & 30.6 & \best{30.0} & \best{0.0} \\
\rowcolor{grayrow}
RepBend & Trained & \good{0.1} & \good{0.6} & 38.8 & \good{0.2} \\
Triplet & Trained & \best{0.0} & \best{0.4} & 39.8 & \best{0.0} \\
\rowcolor{grayrow}
Tri-Adv & Trained & 20.9 & 22.2 & 43.6 & 20.8 \\
\midrule
\ours DDO & Ours & \good{1.2} & \good{1.7} & 37.5 & \best{0.0} \\
\ours +debias & Ours & \good{1.2} & \best{0.4} & 36.3 & \best{0.0} \\
\bottomrule
\end{tabular}
\end{table}

\paragraph{Prompt-level jailbreaks.}
Although DDO targets mechanistic refusal removal, we also evaluate it against prompt-level attacks that operate purely at the input level (Table~\ref{tab:prompt_attacks}). DDO + debiasing achieves \smash{$0.4\%$} GCG~\citep{zou2023gcg} ASR, matching the best trained defenses, and \smash{$1.2\%$} on HumanJailbreaks. A plausible explanation is that optimization-based prompt attacks like GCG suppress the same low-dimensional refusal feature targeted by RFA~\citep{yu2024refat}; DDO's decoy neurons partially re-inject refusal-correlated signal, making it harder for gradient-based prompt optimization to fully suppress refusal. However, semantic attacks like PAIR~\citep{chao2023pair} (\smash{$36.3\%$} ASR) circumvent refusal through meaning rather than activation geometry, and DDO provides no advantage here. Defending against semantic jailbreaks requires complementary input-level defenses such as input classifiers.

\subsection{Ablations}
\label{sec:ablations}

Table~\ref{tab:ablation} isolates each DDO component on Llama-3-8B-Instruct by removing one element at a time from the full configuration. All ablations preserve MMLU within 3 points of the base model.

\paragraph{Effect of gradient optimization.} Random orthogonal decoys (no optimization) achieve \smash{$72.3\%$} ASR---only a modest reduction from the undefended base (\smash{$85.1\%$}). In contrast, gradient optimization under the full 4-part objective reduces ASR to \smash{$1.8\%$}, confirming that orthogonality alone is insufficient: decoys must be optimized to dominate the harmful--safe contrast that the attacker re-estimates.

\paragraph{Loss components.} The \emph{confusion loss} \smash{$\mathcal{L}_{\text{confusion}}$} and \emph{refusal-score loss} \smash{$\mathcal{L}_{\text{refusal-score}}$} are the most critical terms: removing either increases ASR to \smash{$43.2\%$} and \smash{$41.8\%$} respectively. The confusion loss directly optimizes decoy directions to mislead the DIM estimator, while the refusal-score loss prevents first-token collapse (``Sure \ldots'' followed by refusal). The retain KL term stabilizes optimization (ASR \smash{$22.4\%$} without it) but is less critical on its own.

\begin{table}[t]
\centering
\caption{DDO ablation on Llama-3-8B-Instruct. RFA ASR = mean of 4 variants, 3-judge avg (\%). \cmark{bestcell}{\textbf{best}}\,/\,\cmark{bestcell}{good}\,/\,\cmark{badcell}{poor}\,/\,\cmark{oursrow}{ours}. All components contribute; confusion and refusal-score losses are most critical.}
\label{tab:ablation}
\begin{tabular}{@{}lccc@{}}
\toprule
Ablation & RFA$\downarrow$ & MMLU$\uparrow$ & XST$\uparrow$ \\
\midrule
\rowcolor{grayrow}
Base (no defense) & \bad{85.1} & \best{68.1} & 94.8 \\
Random orth.\ decoys & \bad{72.3} & 67.9 & \best{99.6} \\
\rowcolor{grayrow}
w/o confusion loss & 43.2 & 66.5 & 82.0 \\
w/o refusal-score & 41.8 & 66.7 & \good{90.0} \\
\rowcolor{grayrow}
w/o retain KL & 22.4 & 65.8 & 86.4 \\
Additive (vs.\ replace) & 16.5 & 66.8 & 82.8 \\
\rowcolor{grayrow}
Random neurons & 15.3 & 65.2 & 79.2 \\
\ours \textbf{DDO full} & \best{1.8} & 66.8 & \good{91.6} \\
\bottomrule
\end{tabular}
\end{table}

\paragraph{Neuron selection and compile mode.} Low-norm neuron selection outperforms random neuron choice (ASR \smash{$1.8\%$} vs.\ \smash{$15.3\%$}), supporting our design choice to edit low-impact units to reduce utility disruption while maintaining decoy strength. On Llama-3, replace mode outperforms additive (ASR \smash{$1.8\%$} vs.\ \smash{$16.5\%$}), but this preference is model-specific (Appendix~\ref{sec:compile_mode_details}). Finally, layer placement upstream of the causal refusal zone is critical on models with localized refusal (Appendix~\ref{sec:layer_placement_details}).

\paragraph{Alternative defense mechanisms.} We evaluated 14 additional post-hoc mechanisms beyond DDO (Appendix Table~\ref{tab:mechanism_catalogue}). Linear defenses are brittle under re-estimation: Decoy Shear Transform (\smash{$92\%$} ASR), Refusal Direction Rotation (\smash{$68\%$}), and Representation Rerouting (\smash{$86\%$}) are all defeated once the attacker recomputes DIM on the defended checkpoint. Aggressive edits (e.g., LM-Head Row Scaling) can drive RFA ASR near zero but at catastrophic utility cost (XSTest \smash{$8\%$}, MT-Bench \smash{$1.72$}). Finally, spreading decoy energy across more directions involves a tradeoff predicted by Corollary~\ref{cor:flat_spectrum}: diversified readers weaken rank-1 protection (ASR rises from 10\% to 27\% as the number of decoy groups $K$ increases from 1 to 4) but strengthen higher-rank robustness (rank-16 ASR drops from 5\% to 2\%). The optimal $K$ depends on the expected attacker rank budget (Appendix~\ref{sec:spectral_validation}).

\paragraph{Cross-architecture mechanism comparison.} We evaluated 8 post-hoc mechanisms across all 7 models (10--30 Optuna trials each; Appendix~\ref{sec:all_mechanisms}). DDO is the only single mechanism that achieves \smash{$<10\%$} ASR reliably across all architectures. Other mechanisms are individually weak (e.g., gate\_boost: \smash{$56\text{--}81\%$} ASR) or break generation on most models (diverse Q-head routing, rotation, V-projection, head amplification, and KV strengthening all fail the coherence check on Yi, Gemma-2, and GLM-4). Even random orthogonal decoys without gradient optimization fail on 5 of 7 models (\smash{$18\text{--}53\%$} ASR). This supports the core DDO design: nonlinear SwiGLU gating plus gradient optimization (and per-model compile-mode selection) are jointly necessary for cross-architecture robustness. Full mechanism definitions are in Appendix~\ref{sec:appendix_mech_defs}.

\subsection{Discussion}
\label{sec:analysis}

\paragraph{The role of nonlinearity.}
Adaptive multi-phase RFA repeatedly estimates a linear DIM direction on the defended checkpoint. DDO's SwiGLU gate makes the decoy response depend on the prompt: writing \smash{$a \defeq \ip{\hat{\mathbf{r}}_\ell}{\tilde{\mathbf{x}}_\ell}$}, the decoy amplitude scales as \smash{$\silu(\beta a)\cdot a$} and behaves like \smash{$\beta a^2$} for large positive \smash{$a$}. The prompt-dependent decoy response can bias the DIM estimate toward the decoy subspace. Theorems~\ref{thm:estimator_corruption} and~\ref{thm:spectral_toll} characterize the resulting subspace overlap under their stated assumptions.

\paragraph{Resistance to DDO undo attacks.}
A defense-aware attacker may try to \emph{undo} DDO by identifying edited neurons. Let \smash{$\hat{\mathbf{r}}_{\mathrm{clean}}$} denote the reference refusal direction and \smash{$\hat{\mathbf{d}}_{\mathrm{atk}}$} the DIM direction the attacker estimates on the defended checkpoint. The attacker scores each MLP gate row by \smash{$s_i \defeq |\ip{\mathbf{w}_{\mathrm{gate}}^{(i)}}{\hat{\mathbf{d}}_{\mathrm{atk}}}|$} and zeros the down-projection columns corresponding to the highest-scoring rows. This would work if \smash{$\hat{\mathbf{d}}_{\mathrm{atk}} \approx \hat{\mathbf{r}}_{\mathrm{clean}}$}, since DDO writes \smash{$\hat{\mathbf{r}}_{\mathrm{clean}}$}-aligned triggers into gate rows. Instead, estimator corruption rotates \smash{$\hat{\mathbf{d}}_{\mathrm{atk}}$} away from \smash{$\hat{\mathbf{r}}_{\mathrm{clean}}$}, so edited neurons are not salient under this score. Across Qwen3, Mistral-7B, and Yi, mid-zone target layers satisfy \smash{$\cossim(\hat{\mathbf{r}}_{\mathrm{clean}}, \hat{\mathbf{d}}_{\mathrm{atk}}) \approx 0$} and the modified rows' scores fall within the natural top-5 distribution. The evaluated undo heuristics provide partial recovery, with ASR remaining below the undefended model in Table~\ref{tab:undo_asr} (Appendix~\ref{sec:detection_analysis}). ASR ranges from 6\% to 12\% without a sustained increase as the attacker's probe budget grows from 32 to 1024 prompts per class (Appendix~\ref{sec:probe_sweep}). Under rank-$k$ SVD attacks, DDO with a single reader is effective at rank 1 (4\% ASR) but degrades at higher ranks (39\% at $k{=}16$; Appendix~\ref{sec:rank_sweep}). Diversifying readers across $K$ groups extends higher-rank protection at the cost of rank-1 strength, as predicted by Corollary~\ref{cor:flat_spectrum} (Appendix~\ref{sec:spectral_validation}).

\paragraph{Comparison with trained defenses.}
On standard RFA, DDO matches the strongest trained baselines across five model families (Table~\ref{tab:cross_baseline}) at \smash{$30\text{--}450\times$} lower cost per configuration. Under Heretic on Llama-3, DDO alone (\smash{$18\%$}) remains weaker than RepBend~\citep{yousefpour2025repbend} (\smash{$1.3\%$}). A plausible explanation is \emph{entanglement}: gradient-based training can distribute refusal across many parameters jointly, while closed-form edits are easier to localize and therefore easier to target with weight-level attacks.

\section{Conclusion and Limitations}
\label{sec:conclusion}

Existing defenses against refusal ablation require gradient-based training per checkpoint, creating a bottleneck for the open-weight release cycle. DDO demonstrates that a mechanistic alternative---repurposing MLP neurons to corrupt the attacker's contrastive estimator---can match trained baselines on standard RFA across six model families at \smash{$30\text{--}450\times$} lower cost per configuration. The key insight is that contrastive attacks are only as good as the contrast they estimate: by injecting a large, harmful-selective, refusal-orthogonal decoy signal, DDO forces the attacker to ablate decoys rather than genuine refusal.  \\
DDO requires only a small generic probe set (128 harmful + 128 safe prompts), gradient optimization of the decoy parameters (base weights frozen, so memory overhead beyond inference is minimal), and a single GPU---no full finetuning infrastructure, no method-specific safety datasets, and few design decisions beyond compile mode and layer range. This makes it applicable by any party with access to the weights: model providers before release, downstream deployers, safety auditors, or automated agents. The feedback loop is fast (minutes per optimization run) and defense strength is controllable via compile mode and scalar gains (\smash{$\beta, s$}), making DDO well-suited to agent-driven safety pipelines that iteratively harden and evaluate checkpoints.

\paragraph{Limitations and future work.}
DDO improves resistance to refusal ablation, but sustained adaptive re-estimation remains a challenge for all evaluated defenses. DDO can also increase over-refusal on some models, although orthogonal debiasing and compile-mode selection mitigate this effect (Appendix~\ref{sec:limitations}). Future work could investigate post-hoc edits that distribute refusal behavior more broadly across parameters to improve robustness to weight-level attacks. Evaluating DDO alongside inference-time guardrails, such as Llama Guard, would help establish whether these approaches provide complementary protection.

\section*{Acknowledgments}
The authors thank Tatiana Gaintseva, Rebecca Portnoff, and the team at THORN for valuable discussions. Aashiq Muhamed is grateful for support from the Amazon AI Ph.D.\ Fellowship, the Cooperative AI PhD Fellowship, the ML Alignment and Theory Scholars (MATS) Program, and the Supervised Program for Alignment Research (SPAR).

\bibliographystyle{unsrtnat}
\bibliography{references}

\clearpage
\appendix

\section{Limitations and Future Work}
\label{sec:limitations}

Our evaluation covers three attack tiers and compares DDO with trained baselines across six primary model families, including six trained defenses on Llama-3. We also evaluate DDO alone on seven additional models. These experiments identify several limitations and directions for further study.

\paragraph{Adaptive attackers.}
All defenses in Table~\ref{tab:adaptive} reach high ASR under sustained re-estimation within the eight-phase budget. DDO (rank~\smash{$8$}) reaches a worst-case ASR of \smash{$65\%$} (LlamaGuard-2), compared with \smash{$58\%$} for RepBend. Evaluations with longer phase budgets, joint multi-direction projections, and estimators beyond DIM would help establish how robustness changes under stronger adaptive attacks.

\paragraph{Weight-level attacks.}
DDO achieves \smash{$18\%$} ASR under Heretic on Llama-3, compared with \smash{$24.0\%$} for Circuit Breakers, \smash{$82.3\%$} for LAT, and \smash{$85.3\%$} for ReFAT (Table~\ref{tab:main}). RepBend~\citep{yousefpour2025repbend} achieves a lower ASR of \smash{$1.3\%$} through gradient-based parameter--behavior entanglement that is difficult to replicate with post-hoc edits alone. Closing this gap---e.g., via post-hoc edits that induce deeper entanglement surviving weight-level optimization---is a promising direction for future work.

\paragraph{Over-refusal on some models.}
DDO can reduce compliance with benign requests on some models. Similar or larger reductions occur with trained defenses, including LAT and ReFAT on Llama-3 (Table~\ref{tab:main}), indicating that over-refusal is a broader challenge for safety hardening. Orthogonal debiasing can restore benign compliance while preserving RFA robustness (Section~\ref{sec:compositions}), and additive compilation reduces over-refusal on some architectures. Further work should investigate how to select these mitigations automatically for each model.

\paragraph{Compile mode selection.}
Choosing between replace and additive compilation currently requires evaluating both modes and comparing ASR and XSTest scores. This search is inexpensive relative to baseline finetuning, but a predictive rule based on properties such as refusal redundancy could reduce the need for repeated evaluations.

\subsection{Broader Impacts}

\paragraph{Misuse of open-weight models.}
Training-free refusal removal has moved well beyond academic proof-of-concept. The mechanistic result that refusal is mediated by a low-dimensional residual-stream direction~\citep{arditi2024refusal} was rapidly operationalized into public tooling: community tutorials describe one-command uncensoring without retraining~\citep{labonne2024abliteration}, and automated tools such as Heretic apply hyperparameter-optimized weight edits on a consumer GPU with no ML expertise required~\citep{weidmann2025heretic}. A recent large-scale audit of public model repositories found over 8{,}600 safety-modified open-weight checkpoints across 1{,}300 namespaces, with more than half packaged in GGUF format optimized for local consumer-hardware deployment and accumulating tens of millions of tracked downloads~\citep{sokhansanj2025uncensored}. Refusal removal has therefore become a distribution and hosting problem: users can either run abliteration scripts themselves, or simply download an already-modified checkpoint.

\paragraph{Deployment considerations.}
This distribution surface is exactly where DDO intervenes. Once aligned weights are publicly downloadable, there is no practical way to guarantee that all copies receive safety updates or can be rolled back~\citep{casper2025openweight}. Per-checkpoint safety retraining cannot address this at scale: it is too slow and too expensive to apply to every community finetune, merge, or re-quantized derivative hosted on a platform. DDO---requiring approximately two minutes per optimization run on a single GPU without base-model finetuning---can be applied by any party in the distribution chain before serving or redistributing a checkpoint. This is consistent with emerging policy guidance: NIST explicitly assigns monitoring and mitigation responsibilities to model-hosting platforms and other distribution channels~\citep{nist2025misuse}. By raising the cost of the most accessible refusal-removal workflows, DDO addresses a gap that neither upstream safety training nor downstream prompt filtering can close alone.

\paragraph{Misuse domains.}
The impact of increasing the attacker's cost is largest in domains where compliant local assistance lowers barriers to serious misuse. For CBRN and biosecurity, LLM access has been shown to substantially accelerate novices on dual-use biology tasks: a controlled uplift study found that participants with LLM assistance were $4.16\times$ more accurate than internet-only controls, with 89.6\% reporting little difficulty obtaining dual-use information despite safeguards~\citep{zhang2026noviceuplift}. Pre-deployment evaluation of high-consequence biological capabilities is now considered a baseline safety requirement~\citep{pannu2025dualusebio}, and WMDP provides a systematic benchmark for hazardous biosecurity and chemical knowledge in LLMs~\citep{li2024wmdp}. For child safety, robust refusal reduces the availability of uncensored local assistants that can be used to facilitate grooming, coercion, or integration into abuse pipelines. Academic analysis identifies AI-generated CSAM as enabling revictimization, normalization, and lowered barriers to offending~\citep{ociardha2026aicsam}, while operational reporting documents sharply rising volumes and calls for safety-by-design action by platform providers~\citep{iwf2025aicsam}. In both domains, DDO does not provide complete protection---determined adversaries with sufficient resources can still bypass it---but it raises the cost of the most accessible and widespread attack workflows.

\paragraph{Potential negative impacts and dual-use concerns.}
Three dual-use risks merit discussion. First, our adaptive multi-phase RFA attack is a new, stronger variant of existing RFA; however, it is a natural extension of publicly available methods, and we believe the defensive contribution outweighs this incremental offensive capability. Second, our mechanistic analysis of why defenses fail (nonlinearity requirements, subspace overlap bounds, neuron detectability) could inform adversaries seeking to circumvent safety measures; our analysis and released evaluation code are intended to support defensive design and robustness assessment. Third, DDO should be understood as a cost-raising hardening layer---not a complete safety guarantee---and we report its residual attack surfaces transparently in Section~\ref{sec:limitations}.

\paragraph{Responsible disclosure.}
We release DDO as a defensive hardening tool, together with code for evaluating its robustness to refusal-removal attacks. We do not release new uncensored model checkpoints or finetuning recipes for removing safety guardrails. All evaluation uses established public benchmarks (JailbreakBench, HarmBench, AdvBench) and models derived from publicly available instruction-tuned checkpoints.
\section{Extended Related Work}
\label{sec:extended_related}

\paragraph{Mechanistic refusal representations and abliteration.}
Several analyses of RLHF-aligned LLMs show that refusal can be removed by ablating low-dimensional residual-stream features. Projecting out a rank-1 difference-in-means (DIM) direction---the \emph{refusal feature}---often eliminates refusal, enabling Refusal Feature Ablation (RFA) and ``abliteration''~\citep{arditi2024refusal}. Such direction-based interventions have been contextualized under the Linear Representation Hypothesis~\citep{zou2023representation}. Subsequent work argues that refusal geometry can be multi-dimensional (e.g., \emph{concept cones}~\citep{wollschlager2025geometry} and multiple refusal-mediating directions~\citep{joad2026beyond}), with orthogonal safety dimensions~\citep{pan2025hiddendimensions} and evidence that harmfulness and refusal are encoded separately~\citep{zhao2025harmfulnessrefusal}. Complementary mechanistic studies localize safety-relevant structure to specific neurons~\citep{chen2025safetyneurons,li2024safetylayers} and intermediate hidden states~\citep{zhou2024alignment}, while sparse autoencoders reveal refusal as a composition of fine-grained latent features~\citep{yeo2025saerefusal,obrien2024saerefusal}. \citet{jain2024safetyfinetuning} show that safety finetuning induces narrow, localized internal changes, which helps explain why contrastive estimators like RFA can locate and remove the refusal mechanism---motivating defenses that corrupt the estimator itself. We include both single-phase RFA and adaptive multi-phase RFA in our attack ladder (Section~\ref{sec:attacks}).

\paragraph{Defenses against refusal ablation.}
Existing defenses against RFA are largely training-time interventions. Circuit Breakers~\citep{zou2024circuitbreakers} finetunes checkpoints with a representation-rerouting objective; LAT~\citep{casper2024lat}, targeted LAT~\citep{sheshadri2024targetedlat}, and ReFAT~\citep{yu2024refat} rely on adversarial finetuning; and representation-reshaping approaches such as RepBend~\citep{yousefpour2025repbend} and contrastive objectives~\citep{simko2025crl} modify refusal geometry during gradient-based training. A concurrent defense targets abliteration directly~\citep{shairah2025embarrassingly}, but still requires finetuning. While effective, these approaches require per-checkpoint training runs and, depending on the method, additional data collection or attack pipelines (e.g., adversarial examples) and hyperparameter tuning. In contrast, we study \emph{post-hoc mechanistic weight editing}: decoy parameters optimized using a small prompt set and compiled into model weights with no architectural or runtime overhead, enabling rapid hardening and composition across open-weight releases.
Because these edits target specific components (layers, matrices, neurons), they also provide an interpretable and automatable complement to training-based hardening. Hyperparameter search with Optuna selects the configuration of these edits, which can also serve as an initialization for subsequent finetuning.

\paragraph{Prompt-level jailbreaks and robust refusal.}
LLMs are also vulnerable to prompt-based jailbreaks that circumvent refusal by optimizing adversarial inputs, without modifying model weights. We report three widely used automated attacks---GCG~\citep{zou2023gcg}, PAIR~\citep{chao2023pair}, and AutoDAN~\citep{liu2024autodan}---under HarmBench's standardized robust-refusal evaluation setting~\citep{mazeika2024harmbench}. ReFAT provides evidence that many prompt jailbreaks can suppress a low-dimensional residual-stream refusal feature~\citep{yu2024refat}, linking these input-space attacks to the same refusal representation exploited by RFA; we therefore include them as a complementary evaluation of how our post-hoc weight edits affect robustness under input-space attack.

\paragraph{Over-refusal calibration.}
Over-refusal---where aligned models refuse benign requests---has been studied through dedicated benchmarks~\citep{cui2025orbench} and addressed via representation-level interventions. A single learned direction can calibrate false refusal in a training-free manner~\citep{wang2024falserefusal}, while targeted representation finetuning can shift refusal boundaries with minimal utility loss~\citep{dabas2025justenough}. Our orthogonal debiasing mechanism (Section~\ref{sec:compositions}) is complementary: it projects out an over-refusal direction from weight matrices after DDO, restoring benign compliance while preserving robustness.

\paragraph{Real-world prevalence of abliteration.}
The refusal direction exploited by RFA transfers across languages~\citep{wang2025multilingualrefusal}, broadening the scope of the vulnerability beyond English-only settings. A large-scale audit found over 8{,}600 uncensored model repositories on HuggingFace, with abliterated variants achieving \smash{${\sim}74\text{--}80\%$} compliance on unsafe prompts~\citep{sokhansanj2025uncensored}. The growing ecosystem of automated abliteration tools across architectures has been surveyed~\citep{young2025abliterationtools}, and cheap post-hoc weight editing has been argued to pose serious safety risks for open-weight releases~\citep{youssef2025editingrisks}. These considerations motivate our focus on post-hoc hardening primitives that can be applied rapidly across checkpoints.

\section{Notation and Background}
\label{sec:background}

We use a standard decoder-only transformer architecture (e.g., Llama-3~\citep{grattafiori2024llama3}). For a fixed token position (we omit token indices), let \smash{$\mathbf{h}_\ell \in \R^d$} denote the residual-stream state at the input of layer \smash{$\ell$} (hidden dimension \smash{$d$}). Each layer is \emph{pre-norm}: RMSNorm is applied before attention and before the MLP, and each sublayer update is added back to the residual stream. We write RMSNorm with gain \smash{$\boldsymbol{\gamma}_\ell \in \R^d$} as \smash{$\mathrm{RMSNorm}(\mathbf{x}) \defeq \mathbf{x}/\mathrm{RMS}(\mathbf{x}) \odot \boldsymbol{\gamma}_\ell$}, where \smash{$\mathrm{RMS}(\mathbf{x}) \defeq \sqrt{\frac{1}{d}\sum_{i=1}^d x_i^2 + \varepsilon}$}.

\paragraph{Attention.} Let $\mathbf{x}_\ell \defeq \mathrm{RMSNorm}(\mathbf{h}_\ell)$. Multi-head attention computes queries, keys, and values via learned projections:
\begin{equation}
\mathbf{Q}_\ell = \mathbf{W}_q \mathbf{x}_\ell,\quad
\mathbf{K}_\ell = \mathbf{W}_k \mathbf{x}_\ell,\quad
\mathbf{V}_\ell = \mathbf{W}_v \mathbf{x}_\ell,
\end{equation}
Multi-head attention applies rotary positional embeddings (RoPE) to $\mathbf{Q}_\ell,\mathbf{K}_\ell$ and produces an output $\mathbf{a}_\ell$ projected back to the residual stream:
\begin{equation}
\mathbf{a}_\ell \defeq \mathbf{W}_o\,\mathrm{Attn}(\mathbf{Q}_\ell,\mathbf{K}_\ell,\mathbf{V}_\ell),\qquad
\mathbf{h}_\ell^{\mathrm{attn}} \defeq \mathbf{h}_\ell + \mathbf{a}_\ell.
\end{equation}

\paragraph{MLP (SwiGLU).} Let \smash{$\tilde{\mathbf{x}}_\ell \defeq \mathrm{RMSNorm}(\mathbf{h}_\ell^{\mathrm{attn}})$}. Many modern LLMs (including Llama-3) use a SwiGLU MLP~\citep{shazeer2020glu} with gate/up projections \smash{$\mathbf{W}_{\mathrm{gate}},\mathbf{W}_{\mathrm{up}} \in \R^{d_{\mathrm{inter}}\times d}$} and down projection \smash{$\mathbf{W}_{\mathrm{down}} \in \R^{d\times d_{\mathrm{inter}}}$}:
\begin{equation}
\begin{aligned}
\mathbf{m}_\ell &\defeq \mathbf{W}_{\mathrm{down}}\big(\silu(\mathbf{W}_{\mathrm{gate}}\tilde{\mathbf{x}}_\ell)\odot(\mathbf{W}_{\mathrm{up}}\tilde{\mathbf{x}}_\ell)\big),\\
\mathbf{h}_{\ell+1} &\defeq \mathbf{h}_\ell^{\mathrm{attn}} + \mathbf{m}_\ell,
\end{aligned}
\end{equation}
where \smash{$d_{\mathrm{inter}}$} is the MLP intermediate dimension and \smash{$\odot$} is elementwise multiplication. Some models (e.g., Gemma-2) use GeGLU, which replaces \smash{$\silu$} with \smash{$\mathrm{gelu}$}; our method applies to both variants.

\paragraph{Edited parameters.} Our defenses perform \emph{offline} edits to selected parameter matrices, including \smash{$\mathbf{W}_q,\mathbf{W}_k,\mathbf{W}_v,\mathbf{W}_o,\mathbf{W}_{\mathrm{gate}},\mathbf{W}_{\mathrm{up}},\mathbf{W}_{\mathrm{down}}$}, the embedding matrix \smash{$\mathbf{W}_{\mathrm{emb}}$}, the output head \smash{$\mathbf{W}_{\mathrm{lm}}$}, and (for normalization-based mechanisms) RMSNorm gains \smash{$\{\boldsymbol{\gamma}_\ell\}$}. We fold all edits into these weights before deployment, so defended models run with no architectural or runtime overhead.

\paragraph{Hookpoints (for attacks).} When defining RFA variants (Section~\ref{sec:attacks}), we use three hookpoints per layer: residual-stream input \smash{$\mathbf{h}_\ell$}, attention update \smash{$\mathbf{a}_\ell$}, and MLP update \smash{$\mathbf{m}_\ell$}.

\section{Additional Method Details}

\subsection{First-Token Collapse}
\label{sec:first_token_collapse}

During DDO optimization, the cross-entropy refusal loss $\mathcal{L}_{\text{refuse}}$ anchors safety by training the model to produce refusal continuations on harmful prompts. However, this loss operates on the full sequence and can be satisfied by degenerate solutions: the model outputs a compliance-prefixed first token (e.g., ``Sure'') to reduce the decoy-confusion loss, then pivots mid-sentence to refusal (e.g., ``Sure, I'd be happy to\ldots actually, I cannot assist with that request''). This ``first-token collapse'' is problematic because (i) autoregressive generation commits to the first token, so a compliance prefix can cascade into full compliance under greedy or sampled decoding, and (ii) ASR judges that inspect only the first few tokens may classify such responses as compliant even though the model eventually refuses.

The refusal-score loss $\mathcal{L}_{\text{refusal-score}}$ addresses this by directly supervising first-token logits:
\begin{equation}
\mathcal{L}_{\text{refusal-score}} = \max\!\left(0,\; m - \left[\max_{t \in \mathcal{T}_{\text{ref}}} z_t - \max_{t \in \mathcal{T}_{\text{comp}}} z_t\right]\right),
\end{equation}
where $z_t$ is the logit for token $t$, $\mathcal{T}_{\text{ref}} = \{\texttt{I}, \texttt{Sorry}, \texttt{cannot}\}$ are refusal-prefixed tokens, $\mathcal{T}_{\text{comp}} = \{\texttt{Sure}, \texttt{Here}\}$ are compliance-prefixed tokens, and $m$ is a margin ($m{=}0$ in all experiments). This hinge loss ensures that refusal tokens dominate compliance tokens at the first position. In our ablation (Table~\ref{tab:ablation}), removing $\mathcal{L}_{\text{refusal-score}}$ increases ASR from 1.8\% to 41.8\%, confirming that first-token anchoring is essential for DDO's effectiveness.

\subsection{All Mechanism Definitions}
\label{sec:appendix_mech_defs}

Section~\ref{sec:mechanisms} defines DDO (the core mechanism). Here we define the \smash{$14$} additional mechanisms that appear in at least one reported configuration, ablation, or negative-result analysis. Throughout, let \smash{$\mathbf{x}_\ell \defeq \mathrm{RMSNorm}(\mathbf{h}_\ell)$} and \smash{$\tilde{\mathbf{x}}_\ell \defeq \mathrm{RMSNorm}(\mathbf{h}_\ell^{\mathrm{attn}})$}.

\paragraph{Decoy Shear Transform (linear baseline).} We apply a linear shear that preserves the refusal coordinate but injects orthogonal decoy components: \smash{$\mathbf{M}_\ell \defeq \Id + \sum_{j} c_j\,\mathbf{u}_{\ell,j}\hat{\mathbf{r}}_\ell^\top$} with \smash{$\mathbf{u}_{\ell,j} \perp \hat{\mathbf{r}}_\ell$}, folded into \smash{$\mathbf{W}_o$} and \smash{$\mathbf{W}_{\mathrm{down}}$} by left-multiplication. When harmful prompts induce larger \smash{$\ip{\hat{\mathbf{r}}_\ell}{\mathbf{h}}$} than safe prompts, this shear adds a harmful-selective shift in decoy directions and can rotate the initial DIM estimate away from the causal refusal feature. However, because the transformation is linear, an adaptive attacker can re-estimate DIM on the defended checkpoint and recover the new contrast; on Llama-3 this yields \smash{$92\%$} ASR (3-point RFA, LlamaGuard-2).

\paragraph{Diverse Q-Head Routing.} For each attention head \smash{$h$}, we define a head-specific trigger \smash{$\hat{\mathbf{t}}_{\ell,h} = \cos(\theta)\hat{\mathbf{r}}_\ell + \sin(\theta)\hat{\mathbf{o}}_{\ell,h}$} with \smash{$\hat{\mathbf{o}}_{\ell,h} \perp \hat{\mathbf{r}}_\ell$}, and apply the rank-1 update
\begin{equation}
\mathbf{W}^{(h)}_{q,\ell} \leftarrow \mathbf{W}^{(h)}_{q,\ell} + \alpha(\mathbf{W}^{(h)}_{q,\ell}\hat{\mathbf{t}}_{\ell,h})\hat{\mathbf{t}}_{\ell,h}^\top.
\end{equation}
This makes different heads respond to slightly different rotations of refusal, aiming to distribute harmful--safe separation across head-specific subspaces. Since attention includes a softmax nonlinearity, such diversification could in principle make the overall behavior harder to summarize by a single residual-stream DIM direction; in practice the edit is brittle and often fails the coherence check on some architectures.

\paragraph{RoPE Dimension Amplification.} We rescale selected rotary frequencies \smash{$\omega_{\ell,j} \leftarrow \rho\,\omega_{\ell,j}$} for \smash{$j \in \mathcal{J}$}, \smash{$\rho > 1$}, by editing the RoPE parameters used by \smash{$\mathbf{W}_q$} and \smash{$\mathbf{W}_k$}. If refusal-related computation relies on consistent attention patterns across positions, increasing rotation rates in a subset of dimensions can induce more position-dependent variation, potentially smearing harmful--safe differences across tokens and weakening a simple DIM estimator. This mechanism is highly architecture-sensitive and often degrades generation.

\paragraph{Value-Projection Conditioning.} We add a rank-1 component to the value projection, \smash{$\mathbf{W}_{v,\ell} \leftarrow \mathbf{W}_{v,\ell} + \alpha\,\hat{\mathbf{u}}_\ell\hat{\mathbf{r}}_\ell^\top$}, where \smash{$\hat{\mathbf{u}}_\ell$} is a fixed random unit vector. This explicitly routes the refusal coordinate \smash{$\ip{\hat{\mathbf{r}}_\ell}{\mathbf{x}}$} into the value stream, so that even if an attacker partially removes refusal by ablating a mis-estimated direction, any surviving refusal signal is amplified and broadcast through attention aggregation downstream. This edit is attention-sensitive and can break coherence on some models.

\paragraph{KV Rank-1 Strengthening.} This uses the same update as Value-Projection Conditioning but with a larger coefficient \smash{$\eta$}, applied only in late layers (L26--31 on Llama-3) where refusal is strongest: \smash{$\mathbf{W}_{v,\ell} \leftarrow \mathbf{W}_{v,\ell} + \eta\,\hat{\mathbf{u}}_\ell\hat{\mathbf{r}}_\ell^\top$}. Concentrating the coupling where refusal is already causally active makes small ablation errors more consequential, because any residual refusal coordinate is immediately injected into the value stream in the final computation stages.

\paragraph{Head Amplification.} We score heads by \smash{$s_{\ell,h} = \norm{\hat{\mathbf{r}}_\ell^\top \mathbf{W}^{(h)}_{o,\ell}}$}, select the top-\smash{$k$}, and scale those head outputs by a factor \smash{$\gamma$} (e.g., \smash{$k{=}4$}, \smash{$\gamma{=}2$}). If a subset of heads carries disproportionate refusal-aligned signal, amplifying them increases the magnitude of refusal-correlated residual updates, so a fixed-strength (or mis-targeted) ablation leaves more refusal behind. Because it edits \smash{$\mathbf{W}_o$} (part of Heretic's weight-edit surface), this mechanism is mostly a control for standard RFA rather than a robust defense.

\paragraph{Autoregressive Embedding Boost.} For refusal-prefix tokens \smash{$t \in \{\texttt{I}, \texttt{cannot}, \texttt{can}, \texttt{'t}, \texttt{Sorry}\}$}, we add a small refusal-aligned component to their embeddings: \smash{$\mathbf{W}_{\mathrm{emb}}[t,:] \mathrel{+}= \alpha\,\hat{\mathbf{r}}_{\mathrm{avg}}$}, where \smash{$\hat{\mathbf{r}}_{\mathrm{avg}}$} is the layer-averaged refusal direction. Once generation begins with a refusal prefix, the boosted embedding injects extra refusal signal at the next step, creating a positive feedback loop that can make refusal harder to interrupt if an attacker has only partially removed refusal features.

\paragraph{RMSNorm Gain Control.} We increase RMSNorm gain on the top-\smash{$k$} refusal-aligned channels: \smash{$\gamma_{\ell,j} \leftarrow g\,\gamma_{\ell,j}$} for indices \smash{$j$} with the largest \smash{$|\hat{\mathbf{r}}_{\ell,j}|$}. RMSNorm applies state-dependent rescaling (\smash{$\mathbf{y} = \mathbf{x}/\mathrm{RMS}(\mathbf{x}) \odot \boldsymbol{\gamma}$}), so amplifying refusal-aligned channels changes the effective coupling between channels in a data-dependent way. This could make harmful--safe differences less well described by a fixed linear direction; in practice, such edits often destabilize generation.

\paragraph{Orthogonal Debiasing (utility repair).} We estimate an over-refusal direction \smash{$\hat{\mathbf{v}}$} from benign prompts the model incorrectly refuses (constructed to be approximately orthogonal to \smash{$\hat{\mathbf{r}}_\ell$}), and remove sensitivity to \smash{$\hat{\mathbf{v}}$} by projecting it out of selected matrices: \smash{$\mathbf{W}' = \mathbf{W} - (\mathbf{W}\hat{\mathbf{v}})\hat{\mathbf{v}}^\top$} for input-dimension matrices (\smash{$\mathbf{W}_{\mathrm{emb}}$}), and \smash{$\mathbf{W}' = \mathbf{W} - \hat{\mathbf{v}}(\hat{\mathbf{v}}^\top\mathbf{W})$} for output-dimension matrices (\smash{$\mathbf{W}_o$}, \smash{$\mathbf{W}_{\mathrm{down}}$}). Since DDO can mildly increase benign refusals on some models, removing the dominant benign-bias direction repairs compliance without directly weakening refusal, because the edited direction is chosen to be orthogonal to refusal.

\paragraph{Refusal Direction Rotation.} We apply a fixed Givens rotation in the \smash{$\{\hat{\mathbf{r}}, \mathbf{v}\}$} plane (\smash{$\theta = \pi/3$}, \smash{$\mathbf{v}\perp\hat{\mathbf{r}}$}), folded into \smash{$\mathbf{W}_o$} and \smash{$\mathbf{W}_{\mathrm{down}}$}. This can appear to ``hide'' refusal by moving it away from \smash{$\hat{\mathbf{r}}$}, but the rotated refusal direction \smash{$\mathbf{r}_{\mathrm{def}} = \mathbf{M}\mathbf{r}_{\mathrm{base}}$} is still a single linear direction, so an adaptive attacker can re-estimate DIM on the defended checkpoint and recover it (Llama-3: \smash{$68\%$} ASR).

\paragraph{Representation Rerouting.} We learn a rank-\smash{$16$} update \smash{$\Delta\mathbf{W}$} (applied to \smash{$\mathbf{W}_o$} and \smash{$\mathbf{W}_{\mathrm{down}}$}) that maps harmful activations toward random target representations. If harmful activations are pushed into a more diverse set of directions, the harmful--safe contrast could become less low-rank, weakening rank-1 DIM estimation. In practice, harmful and safe activations remain linearly separable after rerouting, so a re-estimated DIM direction still removes refusal (Llama-3: \smash{$86\%$} ASR).

\paragraph{Refusal Scattering.} We attempt to distribute the refusal direction across \smash{$k{=}8$} random orthogonal residual directions (via rank-\smash{$k$} perturbations to \smash{$\mathbf{W}_o$}) and boost refusal-prefix logits via LM-head scaling. If refusal were spread across multiple orthogonal directions, a rank-1 ablation would remove only a fraction of it, forcing the attacker to use higher rank or more phases. In practice, this approach performed poorly in early tests and we did not evaluate it beyond a coherence check.

\paragraph{LM-Head Multiplicative Row Scaling.} We reshape a small set of logits by multiplicatively rescaling corresponding LM-head rows. For selected token IDs $t$ (refusal-prefix tokens such as ``I/cannot/can'' and compliance tokens such as ``'d/would/Sure/Here''), we apply
\begin{equation}
\mathbf{W}_{\mathrm{lm}}[t,:] \leftarrow s_t\,\mathbf{W}_{\mathrm{lm}}[t,:],
\end{equation}
which rescales \smash{$z_t=\ip{\mathbf{W}_{\mathrm{lm}}[t,:]}{\mathbf{h}}$} without modifying hidden-state activations. Since RFA-like attacks operate on hidden states, changing the LM head can provide additional refusal ``headroom'' even if intermediate refusal features are partially reduced. We tune \smash{$\{s_t\}$} with Optuna~\citep{akiba2019optuna}.

\paragraph{Adversarial Sign Coupling.} We repurpose a small set of SwiGLU units to implement a \emph{bilinear} cross-term between the refusal coordinate and an orthogonal coordinate. In layer $\ell$, sample \smash{$\hat{\mathbf{o}}_\ell \perp \hat{\mathbf{r}}_\ell$} and choose a small set of intermediate units $i$ disjoint from the gated-decoy units (we use a fixed neuron-index offset, 30+). We set each unit so that its gate row aligns with \smash{$\hat{\mathbf{o}}_\ell$} and its up row aligns with \smash{$\hat{\mathbf{r}}_\ell$}, while its down column writes into a fixed random output direction \smash{$\mathbf{u}_{\ell,i}$}. The resulting contribution is
\begin{equation}
\delta_i(\tilde{\mathbf{x}}_\ell) = \silu(\ip{\hat{\mathbf{o}}_\ell}{\tilde{\mathbf{x}}_\ell}) \cdot \ip{\hat{\mathbf{r}}_\ell}{\tilde{\mathbf{x}}_\ell} \cdot \mathbf{u}_{\ell,i}.
\end{equation}
Adaptive multi-phase RFA repeatedly estimates and ablates a \emph{linear} DIM direction at fixed hookpoints. Because the sign-coupled term depends on the \emph{product} of two coordinates, the effective linear surrogate presented by this bilinear feature can shift after ablation and across hookpoints, potentially increasing the number of distinct directions (phases/rank) an attacker must remove to fully suppress refusal-related computation.

\section{Theoretical Proofs and Formalizations}

\subsection{Proofs of the Overlap and Spectral Bounds}
\label{sec:theorem_proof}

\begin{proof}[Proof of Theorem~\ref{thm:estimator_corruption}]
Let $Q_k=Q_k(C)$ and $V_k$ contain the top-$k$ left and right singular vectors of $C$, respectively, and let $\Sigma_k=\operatorname{diag}(\sigma_1(C),\ldots,\sigma_k(C))$. The SVD identity $CV_k=Q_k\Sigma_k$ gives
\[
\Pi_{\mathcal{R}}Q_k=\Pi_{\mathcal{R}}CV_k\Sigma_k^{-1},
\]
where $\Sigma_k$ is invertible because $\sigma_k(C)>0$. Since $Q_k$ and $V_k$ have orthonormal columns,
\[
\begin{aligned}
\|\Pi_{\mathcal{R}}\Pi_k(C)\|_{\mathrm{op}}
&=\|\Pi_{\mathcal{R}}Q_k\|_{\mathrm{op}}\\
&\le \|\Pi_{\mathcal{R}}C\|_{\mathrm{op}}\,
     \|V_k\|_{\mathrm{op}}\,\|\Sigma_k^{-1}\|_{\mathrm{op}}
 = \frac{\|\Pi_{\mathcal{R}}C\|_{\mathrm{op}}}{\sigma_k(C)}.
\end{aligned}
\]
\end{proof}

\begin{proof}[Proof of Theorem~\ref{thm:spectral_toll}]
Because $U^\top U=I_m$, the nonzero singular values of $D_\theta=UA_\theta$ equal those of $A_\theta$. Weyl's singular-value inequality and $\Pi_{\mathcal{R}}U=0$ give
\[
\begin{aligned}
\sigma_k(C_\theta)
&\ge \sigma_k(D_\theta)-\|S_\theta\|_{\mathrm{op}}
 = \sigma_k(A_\theta)-\rho>0,\\
\|\Pi_{\mathcal{R}}C_\theta\|_{\mathrm{op}}
&=\|\Pi_{\mathcal{R}}S_\theta\|_{\mathrm{op}}
 = \rho_{\mathcal{R}}.
\end{aligned}
\]
Substitution into Theorem~\ref{thm:estimator_corruption} proves the bound. If $1\le k\le r_{\mathrm{eff}}(\varepsilon)$, then $\sigma_k(A_\theta)-\rho>\rho_{\mathcal{R}}/\varepsilon$, so the overlap is at most $\varepsilon$.
\end{proof}

\begin{proof}[Proof of Corollary~\ref{cor:flat_spectrum}]
The singular values are nonincreasing, and their squared sum equals the squared Frobenius norm. Hence
\[
k\,\sigma_k(A_\theta)^2
\le \sum_{j=1}^{k}\sigma_j(A_\theta)^2
\le \|A_\theta\|_F^2
\le B^2.
\]
Taking square roots gives $\sigma_k(A_\theta)\le B/\sqrt{k}$. Equality is attained when the first $k$ singular values all equal $B/\sqrt{k}$ and every remaining singular value is zero.
\end{proof}

\begin{corollary}[Overlap under decoy-dominated contrast]
\label{cor:decoy_dominated}
If $C=C_{\mathrm{ref}}+C_{\mathrm{dec}}+E$, with $\Pi_{\mathcal{R}}C_{\mathrm{dec}}=0$, $\|C_{\mathrm{ref}}+E\|_{\mathrm{op}}\le s_{\mathrm{ref}}$, and $\sigma_k(C_{\mathrm{dec}})\ge s_{\mathrm{dec}}>s_{\mathrm{ref}}$, then
\begin{equation}
\left\|\Pi_{\mathcal{R}}\Pi_k(C)\right\|_{\mathrm{op}} \le \frac{s_{\mathrm{ref}}}{s_{\mathrm{dec}}-s_{\mathrm{ref}}}.
\end{equation}
\end{corollary}

\begin{proof}[Proof of Corollary~\ref{cor:decoy_dominated}]
Orthogonality and Weyl's inequality bound the numerator and denominator in Theorem~\ref{thm:estimator_corruption}:
\[
\begin{aligned}
\|\Pi_{\mathcal{R}}C\|_{\mathrm{op}}
&=\|\Pi_{\mathcal{R}}(C_{\mathrm{ref}}+E)\|_{\mathrm{op}}
 \le s_{\mathrm{ref}},\\
\sigma_k(C)
&\ge \sigma_k(C_{\mathrm{dec}})-\|C_{\mathrm{ref}}+E\|_{\mathrm{op}}
 \ge s_{\mathrm{dec}}-s_{\mathrm{ref}}>0.
\end{aligned}
\]
Substituting these bounds proves the claim.
\end{proof}

\subsection{Stackelberg Formulation}
\label{sec:stackelberg}

We formalize the defender--attacker interaction as a capability-constrained Stackelberg game.

\paragraph{Setup.}
The defender commits to a post-training weight edit \smash{$\theta \in \Theta$}, producing checkpoint \smash{$M_\theta$}. After observing \smash{$M_\theta$}, the attacker chooses an attack \smash{$a \in \mathcal{A}$} (e.g., rank-\smash{$k$} RFA, \smash{$P$}-phase adaptive RFA, \smash{$T$}-trial Heretic) to maximize harmful compliance while preserving capability:
\begin{equation}
u_A(\theta, a) = A(\theta, a) - \lambda[\tau - U(\theta, a)]_+ - c(a),
\end{equation}
where $[x]_+\defeq\max(x,0)$, \smash{$A(\theta, a) = \mathrm{ASR}(M_{\theta,a})$}, \smash{$U(\theta,a)$} is a utility proxy (e.g., MT-Bench), \smash{$\tau$} is an attacker utility threshold, \smash{$\lambda\ge0$} controls the capability--compliance tradeoff, and \smash{$c(a)$} is attack cost. The defender solves
\begin{equation}
\theta^\star \in \arg\min_{\theta \in \Theta} \max_{a \in \mathrm{BR}_{\lambda,\tau}(\theta)} \left[A(\theta,a) + \mu[\tau_D - U(\theta, \varnothing)]_+ + C_D(\theta)\right],
\end{equation}
where \smash{$\mathrm{BR}_{\lambda,\tau}(\theta) = \arg\max_{a \in \mathcal{A}} u_A(\theta,a)$} is the attacker's best response, \smash{$\mu$} penalizes defender utility violations, \smash{$\tau_D$} is the defender's utility floor (without attack), and \smash{$C_D(\theta)$} is editing cost.

\begin{definition}[Capability-constrained ASR and deterrence]
\label{def:cc_asr}
For defense \smash{$\theta$}, utility threshold \smash{$\tau$}, and attack family \smash{$\mathcal{A}$}:
\begin{equation}
\mathrm{CC\text{-}ASR}_\tau(\theta) \defeq \max_{a \in \mathcal{A}:\, U(\theta,a) \ge \tau} A(\theta, a).
\end{equation}
We say \smash{$\theta$} is \smash{$(q, \tau)$}\emph{-deterrent} against \smash{$\mathcal{A}$} if \smash{$\mathrm{CC\text{-}ASR}_\tau(\theta) < q$}.
\end{definition}

\noindent CC-ASR separates clean jailbreaks (high ASR with preserved utility) from pyrrhic failures (high ASR with destroyed utility). Worst-case ASR conflates these; CC-ASR isolates the deployment-relevant threat.

\paragraph{Implications for capability-constrained attacks.}
Consider a family of contrastive attacks satisfying the assumptions of Theorem~\ref{thm:spectral_toll}. If every capability-preserving attack in this family has rank at most \smash{$r_{\mathrm{eff}}(\varepsilon)$} (the effective decoy rank), then its causal overlap is at most $\varepsilon$. Assuming further that the increase in expected compliance is at most $L$ times this overlap gives \smash{$\mathrm{CC\text{-}ASR}_\tau(\theta) \le A(\theta, \varnothing) + L\varepsilon$} for that family. Here $L$ is a modeling assumption for a smooth expected compliance rate, rather than discrete empirical ASR. For a given attack rank $k$, DDO therefore seeks to maximize the spectral margin \smash{$\sigma_k(A_\theta) - \rho$} subject to utility constraints.

\paragraph{Stopping criterion for adaptive RFA.}
For adaptive multi-phase RFA, let $A_t$, $U_t$, and $c_t$ denote ASR, utility, and cumulative attack cost after phase $t$. Under a one-step comparison, continuing from phase $t$ to $t+1$ is beneficial when
\[
A_{t+1}-A_t
> \lambda\bigl([\tau-U_{t+1}]_+-[\tau-U_t]_+\bigr)
  +(c_{t+1}-c_t).
\]
The marginal ASR gain must exceed the added capability penalty and attack cost. DDO aims to shift substantial ASR gains to later phases, where capability degradation can outweigh them. Empirically (Figure~\ref{fig:mtbench_adaptive}), DDO (rank \smash{$8$}) maintains MT-Bench \smash{$\geq 5.82$} through Phase~\smash{$8$}, comparable to RepBend (\smash{$6.17$}) and Circuit Breakers (\smash{$5.90$}); ReFAT drops to MT-Bench \smash{$2.46$} at Phase~\smash{$1$}, so its low ASR is bought at the cost of utility under this formulation.

\subsection{Empirical Spectral Analysis}
\label{sec:spectral_validation}

We empirically validate the flat-spectrum tradeoff predicted by Corollary~\ref{cor:flat_spectrum}. We train DDO with $K \in \{1,2,4,8\}$ decoy neurons per layer, each with a \emph{diversified reader} (distinct trigger direction via random orthogonal perturbation). Without diversified readers, all $K$ neurons read the same scalar $\langle \hat{\mathbf{r}}_\ell, \tilde{\mathbf{x}}_\ell \rangle$ and $A_\theta$ is effectively rank 1 regardless of $K$. With diversified readers, each neuron reads a different coordinate, giving $A_\theta$ up to $K$ significant singular values. We then attack each model with rank-$k$ SVD ablation (Table~\ref{tab:spectral_validation}).

\begin{table}[h!]
\centering
\caption{\small Diversified-reader DDO under rank-$k$ attacks (ASR \%$\downarrow$, LlamaGuard-2, Llama-3-8B). $K$ = decoy neurons per layer; $k$ = attack rank. As predicted by Corollary~\ref{cor:flat_spectrum}, increasing $K$ weakens rank-1 protection (energy spread) but strengthens higher-rank protection (more decoy directions to exhaust).}
\label{tab:spectral_validation}
\footnotesize
\begin{tabular}{@{}lcccc@{}}
\toprule
Attack rank $k$ & $K{=}1$ & $K{=}2$ & $K{=}4$ & $K{=}8$ \\
\midrule
\rowcolor{grayrow}
$k{=}1$ & 10 & 21 & 27 & 14 \\
$k{=}2$ & 1 & 1 & 0 & 0 \\
\rowcolor{grayrow}
$k{=}4$ & 1 & 0 & 0 & 1 \\
$k{=}8$ & 2 & 1 & 1 & 8 \\
\rowcolor{grayrow}
$k{=}16$ & 5 & 2 & 2 & 2 \\
\bottomrule
\end{tabular}
\end{table}

Three patterns confirm the corollary: (i) \emph{rank-1 protection weakens with $K$}: ASR rises from 10\% ($K{=}1$) to 27\% ($K{=}4$) as energy is spread across more directions; (ii) \emph{higher-rank protection strengthens}: at $k{=}16$, ASR drops from 5\% ($K{=}1$) to 2\% ($K{\geq}2$); and (iii) $K{=}2$--$4$ provides the best tradeoff (0--1\% ASR at $k{=}2$--$8$ with moderate rank-1 cost).

\paragraph{Empirical singular spectrum.} Table~\ref{tab:singular_spectrum} reports the singular values of $A_\theta$ (the empirical decoy response matrix from Theorem~\ref{thm:spectral_toll}) for three configurations.

\begin{table}[h!]
\centering
\caption{\small Singular values of the decoy response matrix $A_\theta$ on Llama-3-8B DDO. Diversified readers flatten the spectrum ($\sigma_1/\sigma_2$ drops from 934:1 to 15:1), confirming the mechanism underlying Corollary~\ref{cor:flat_spectrum}.}
\label{tab:singular_spectrum}
\footnotesize
\begin{tabular}{@{}lccccc@{}}
\toprule
Configuration & $\sigma_1$ & $\sigma_2$ & $\sigma_3$ & $\sigma_4$ & $\sigma_1/\sigma_2$ \\
\midrule
\rowcolor{grayrow}
$K{=}1$, single reader & 158.8 & 0.17 & 0.16 & 0.04 & 934:1 \\
$K{=}4$, single reader & 339.0 & 2.33 & 1.66 & 1.05 & 145:1 \\
\rowcolor{grayrow}
$K{=}4$, diversified reader & 82.6 & 5.59 & 1.55 & 1.15 & \textbf{15:1} \\
\bottomrule
\end{tabular}
\end{table}

\noindent Diversified readers reduce $\sigma_1/\sigma_2$ from 934:1 (single reader, effective rank 1) to 15:1 (diversified, effective rank $>$1), directly confirming that reader diversification flattens the decoy spectrum as predicted.

\section{Full Experimental Details}

\subsection{Attack Implementation Details}
\label{sec:attack_details}

\paragraph{Standard RFA.}
We compute per-layer DIM directions at the last token position: for each layer \smash{$\ell$}, \smash{$\hat{\mathbf{d}}_\ell = \mathrm{normalize}(\bar{\mathbf{h}}_\ell^{\mathrm{harm}} - \bar{\mathbf{h}}_\ell^{\mathrm{safe}})$} where \smash{$\bar{\mathbf{h}}$} denotes the mean activation over the probe set. Mean activations are accumulated in float64 for numerical stability. All \smash{$32$} per-layer directions are applied simultaneously via forward hooks (three-point: block input + attention output + MLP output; residual-stream: block input only). This differs from the Arditi et al.\ direction \emph{selection} pipeline~\citep{arditi2024refusal}, which selects a single (position, layer) direction by evaluating each candidate using KL divergence on safe prompts. Our per-layer procedure corresponds to a larger ablation budget (one direction per layer) and is therefore at least as strong as the single-direction variant under matched probe sets. For all reported results we use \textbf{unfiltered} DIM estimation (no refusal-score filtering of the probe set), which provides the attacker with a cleaner direction estimate.

\paragraph{Adaptive multi-phase RFA.}
At each phase \smash{$t$}, we collect activations with all prior phases' hooks active, compute fresh per-layer DIM directions, and Gram-Schmidt orthogonalize against all previous directions. All \smash{$t$} direction sets are then ablated simultaneously using three-point hooks. Total hookpoints: up to \smash{$3tL$} per forward pass (\smash{$3$} hookpoints $\times$ \smash{$t$} phases $\times$ \smash{$L{=}32$} layers). The attack is deterministic: three identical runs produce identical direction sequences. We evaluate up to \smash{$8$} phases.

\paragraph{Heretic~\citep{weidmann2025heretic}.}
We run the released Heretic tool in an isolated Python virtual environment via subprocess. Key settings: Optuna TPE sampler with \texttt{n\_startup\_trials=15}, \texttt{n\_ei\_candidates=128}, \texttt{multivariate=True}. Multi-objective optimization: minimize $(\text{refusals}, \text{KL divergence})$. The tool optimizes a peaked per-layer ablation strength profile $\alpha(\ell)$ over $\mathbf{W}_o$ and $\mathbf{W}_{\mathrm{down}}$ matrices. We report results at \smash{$200$} trials (H-200). Direction scope is searched over ``global'' (single direction index) and ``per\_layer''. Per-component search ranges: \texttt{max\_weight} $\in [0.8, 1.5]$, \texttt{max\_weight\_position} $\in [0.6L, L]$, \texttt{min\_weight} $\in [0, 1] \times \texttt{max\_weight}$, \texttt{min\_weight\_distance} $\in [1, 0.6L]$.

\paragraph{Heretic generation quality.}
\label{sec:heretic_quality}
Daggered entries in Tables~\ref{tab:main} and~\ref{tab:full_zoo} indicate post-attack generation degradation associated with KL divergence \smash{$>0.3$} in the selected Heretic trial. Heretic's KL metric measures the change in output probability distributions on benign prompts relative to the checkpoint before attack. It provides a diagnostic of distributional change rather than a direct measure of task accuracy. The utility columns in these tables describe the defended checkpoints before attack; they do not measure post-Heretic utility.

\subsection{DDO and Auxiliary Mechanism Hyperparameters (Llama-3)}
\label{sec:appendix}

Table~\ref{tab:hyperparams} reports exact hyperparameters for DDO and all auxiliary mechanisms on Llama-3. All mechanisms share the same unfiltered DIM estimation protocol and probe budget.

\begin{table}[h!]
\centering
\caption{Hyperparameters for the 15 mechanisms used in reported configurations on Llama-3-8B-Instruct. Shared: bfloat16 weights, float64 direction estimation, \smash{$128$} harmful + \smash{$128$} safe probes, unfiltered DIM, last-token position. DDO (first row) is the core mechanism.}
\label{tab:hyperparams}
\small
\begin{tabularx}{\linewidth}{@{}>{\raggedright\arraybackslash}p{3.3cm}l>{\raggedright\arraybackslash}X@{}}
\toprule
\textbf{Mechanism} & \textbf{Scope} & \textbf{Key parameters} \\
\midrule
\multicolumn{3}{@{}l}{\emph{Core mechanism:}} \\
\rowcolor{grayrow}
\textbf{DDO (Gated Decoy)} & varies & 1 neuron/layer (low\_norm). Orthogonal decoy basis $\mathbf{u}_i \perp \hat{\mathbf{r}}_\ell$, gradient-optimized. Per-model layer ranges and hyperparameters in Table~\ref{tab:ddo_configs}. Modifies $\mathbf{W}_{\mathrm{gate}}, \mathbf{W}_{\mathrm{up}}, \mathbf{W}_{\mathrm{down}}$ \\
\midrule
\multicolumn{3}{@{}l}{\emph{Auxiliary mechanisms (examples; used in some compositions/ablations):}} \\
Value-Projection Conditioning & 12--19 & Rank-1 perturbation of $\mathbf{W}_v$ (random output direction) \\
Head Amplification & 24--28 & Top 4 heads, $2\times$ factor \\
\rowcolor{grayrow}
Orthogonal Debiasing & all & $\mathbf{W}'_{\mathrm{emb}} = \mathbf{W}_{\mathrm{emb}}(\mathbf{I}-\hat{\mathbf{v}}\hat{\mathbf{v}}^\top)$; $\mathbf{W}' = (\mathbf{I}-\hat{\mathbf{v}}\hat{\mathbf{v}}^\top)\mathbf{W}$ for $\mathbf{W}\in\{\mathbf{W}_o,\mathbf{W}_{\mathrm{down}}\}$ \\
\midrule
\multicolumn{3}{@{}l}{\emph{Standalone variants:}} \\
\rowcolor{grayrow}
LM-Head Multiplicative Row Scaling & $\mathbf{W}_{\mathrm{lm}}$ & Optuna TPE, 40 trials, 9 tokens (e.g.\ ``I'': $2.63\times$, ``Sure'': $0.32\times$) \\
Decoy Shear Transform & all & $\mathbf{M} = \mathbf{I} + \mathbf{U}\mathbf{T}\mathbf{U}^\top$; sweep $c{\in}\{1{-}5\}$, $k{\in}\{4,8,16\}$ \\
\rowcolor{grayrow}
Refusal Scattering & varies & Rank-$k$ scatter ($k{=}8$) + LM-head boost \\
Representation Rerouting & varies & Rank-16 $\Delta\mathbf{W}$, $\lambda_{\mathrm{reroute}}$/$\lambda_{\mathrm{preserve}}$ \\
\rowcolor{grayrow}
Refusal Direction Rotation & all & Givens in $\{\hat{\mathbf{r}}, \mathbf{v}\}$ plane, $\theta{=}\pi/3$ \\
\midrule
\multicolumn{3}{@{}l}{\emph{Adaptive-specific additions (used in adaptive compositions):}} \\
\rowcolor{grayrow}
Diverse Q-Head Routing & 16--30 & 16--22 heads, $q_{\mathrm{strength}}$ 0.59--0.75, per-head rotated triggers \\
RoPE Dimension Amplification & 16--28 & 1.45--1.54$\times$ on 11--16 rotary dims \\
\rowcolor{grayrow}
Autoregressive Embedding Boost & $\mathbf{W}_{\mathrm{emb}}$ & Strength 0.68--0.82, tokens \{I, cannot, can, 't, Sorry\} \\
KV Rank-1 Strengthening & 26--31 & Strength 0.55--0.60, rank-1 on $\mathbf{W}_v$ \\
\rowcolor{grayrow}
RMSNorm Gain Control & 16--28 & 1.07--1.30$\times$ on 59--64 refusal-aligned $\gamma$ channels \\
Adversarial Sign Coupling & 4--27 & Gate on $\hat{\mathbf{o}} \perp \hat{\mathbf{r}}$, up on $\hat{\mathbf{r}}$, neuron offsets 30+ \\
\bottomrule
\end{tabularx}
\end{table}

\subsection{Per-Model Causal Refusal Zones}
\label{sec:causal_zones}

To choose DDO layers, we estimate each model's causal refusal zone via single-layer ablations. For each layer \smash{$\ell$}, we ablate only the per-layer DIM direction at \smash{$\ell$} (3-point hooks) and measure the resulting drop in refusal rate (substring heuristic) on \smash{$16$} harmful prompts. We classify a layer as ``causal'' if the refusal drop exceeds \smash{$10\%$}. Table~\ref{tab:causal_zones} reports the resulting causal zones and DDO placement for all \smash{$7$} models. ``Causal Zone'' is the span from the earliest to the latest causal layer; \#C is the number of layers within that span that exceed the \smash{$10\%$} threshold (not necessarily contiguous---e.g., Llama-2 has \smash{$3$} causal layers scattered across L13--31); and Frac is \#C divided by the total number of layers. The upstream placement principle (Section~\ref{sec:upstream}) suggests that, on localized models, DDO layers should start at or before the causal zone (allowing partial overlap).

\begin{table}[h!]
\centering
\caption{Causal refusal zones, DDO placement, and compile mode across all \smash{$7$} models. \#C = number of causal layers; Frac = fraction of total layers that are causal. Compile mode does not follow a simple redundancy rule: Mistral has \smash{$81\%$} causal fraction (distributed) yet requires additive because replace inverts its refusal signal (Table~\ref{tab:refusal_proj}).}
\label{tab:causal_zones}
\small
\setlength{\tabcolsep}{2pt}
\begin{tabular}{@{}llccllc@{}}
\toprule
Model & Causal Zone & \#C & Frac & Pattern & DDO Layers & Mode \\
\midrule
\rowcolor{grayrow}
Yi (48L) & L3--45 & 42 & 88\% & Distributed & L18--32 & Rep \\
Llama-2 (32L) & L13--31 & 3 & 9\% & Sparse & L4--12 & Rep \\
\rowcolor{grayrow}
Llama-3 (32L) & L8--31 & 8 & 25\% & Localized & L6--14 & Rep \\
Gemma-2 (42L) & L12--23 & 12 & 29\% & Localized & L10--18 & Add \\
\rowcolor{grayrow}
Qwen3 (36L) & L20--34 & 15 & 42\% & Localized & L12--16 & Add \\
Mistral (32L) & L2--28 & 26 & 81\% & Distributed & L2--18 & Add \\
\rowcolor{grayrow}
GLM-4 (40L) & L16--22 & 6 & 15\% & Localized & L6--17 & Rep \\
\bottomrule
\end{tabular}
\end{table}

\subsection{Compile Mode Selection Details}
\label{sec:compile_mode_details}

We compile DDO parameters into model weights using one of two modes: (i) \textbf{replace} (\texttt{.copy\_()}), which completely overwrites the neuron's row/column vectors with the decoy trigger and output (a clean, high-magnitude decoy signal, but it removes the neuron's original computation), and (ii) \textbf{additive} (\texttt{.add\_()}), which adds the decoy signal on top of the original neuron weights (preserves original behavior, but yields a weaker, mixed decoy signal).

The preferred mode depends on how each model's refusal circuitry responds to neuron overwriting, rather than a simple redundancy rule:

\begin{table}[!htb]
\centering
\caption{DDO compile mode comparison across 7 models (ASR \%$\downarrow$, mean of 4 RFA variants, 3-judge avg, 15 Optuna trials per mode). \cmark{bestcell}{best}\,/\,\cmark{badcell}{poor}\,/\,\cmark{oursrow}{ours}; $\times$ = coherence failure. The preferred mode is model-specific; replace is stronger when tolerated but breaks some architectures.}
\label{tab:compile_mode}
\small
\begin{tabular}{@{}llccc@{}}
\toprule
Model & Gate & Replace$\downarrow$ & Additive$\downarrow$ & Best \\
\midrule
\rowcolor{grayrow}
Yi-1.5-9B & SwiGLU & \best{5\%} & 18\% & Replace \\
Llama-2-7B & SwiGLU & \best{1\%} & 8\% & Replace \\
\rowcolor{grayrow}
Llama-3-8B & SwiGLU & \best{0\%} & 7\% & Replace \\
Gemma-2-9B & GeGLU & \bad{$\times$} & \best{0\%} & Additive \\
\rowcolor{grayrow}
Qwen3-8B & SwiGLU & 18\% & \best{8\%} & Additive \\
Mistral-7B & SwiGLU & 42\% & \best{0\%} & Additive \\
\rowcolor{grayrow}
GLM-4-9B & Fused & \best{2\%} & 48\% & Replace \\
\bottomrule
\end{tabular}
\end{table}

Replace mode produces a stronger decoy signal but is more disruptive: overwriting a neuron's weights removes its original computation. On some models this eliminates baseline refusal entirely (Mistral: \smash{$3\%$} Ref under replace; Table~\ref{tab:refusal_proj}), while on others the model compensates (Yi: \smash{$93\%$} Ref despite replacement). Gemma-2 (GeGLU) fails the coherence check under replace entirely. Additive mode preserves the original neuron and superposes the decoy, making it safer but producing a weaker, mixed decoy signal. Note that the closed-form decoy contribution (Equation~\ref{eq:decoy_response}) is exact only under replace mode; in additive mode, the SwiGLU nonlinearity produces cross-terms between the original and inserted weights, so Equation~\ref{eq:decoy_response} should be interpreted as the intended decoy component rather than the exact neuron output. In practice, we evaluate both modes and select by ASR subject to passing the coherence check.

\subsection{Layer Placement Analysis}
\label{sec:layer_placement_details}

DDO's effectiveness depends critically on \emph{where} decoys are placed relative to the model's causal refusal zone---the layers whose ablation causally reduces refusal (Table~\ref{tab:causal_zones}). Ideally, decoy layers should start at or before the causal zone so that the decoy signal contaminates the attacker's estimator without disrupting the refusal computation itself.

\begin{table}[!htb]
\centering
\caption{Effect of decoy placement on baseline refusal rate. Refusal rate measured on 100 JailbreakBench behaviors (substring heuristic, without attack). DDO uses per-model Optuna-tuned hyperparams (Table~\ref{tab:ddo_configs}). Ref$_{\text{base}}$ = undefended model; Ref$_{\text{in}}$ = decoys placed inside the causal zone; Ref$_{\text{up}}$ = decoys placed upstream.}
\label{tab:dilution}
\small
\begin{tabular}{@{}lccccc@{}}
\toprule
Model & Causal & Ref$_{\text{base}}\uparrow$ & Ref$_{\text{in}}\uparrow$ & Ref$_{\text{up}}\uparrow$ & Mode \\
\midrule
\rowcolor{grayrow}
Yi & L3--45 & 75\% & 93\% & 93\% & Replace \\
Llama-2 & L13--31 & 96\% & 94\% & 96\% & Replace \\
\rowcolor{grayrow}
Llama-3 & L8--31 & 95\% & 91\% & 95\% & Replace \\
Gemma-2 & L12--23 & 93\% & 68\% & 93\% & Additive \\
\rowcolor{grayrow}
Qwen3 & L20--34 & 37\% & 19\% & 35\% & Additive \\
Mistral & L2--28 & 58\% & 3\% & 18\% & Additive \\
\rowcolor{grayrow}
GLM-4 & L16--22 & 96\% & 72\% & 96\% & Replace \\
\bottomrule
\end{tabular}
\end{table}

Table~\ref{tab:dilution} shows that placement affects baseline refusal: placing decoys \emph{inside} the causal zone can substantially reduce refusal (e.g., Gemma-2: \smash{$93\% \rightarrow 68\%$}; Qwen3: \smash{$37\% \rightarrow 19\%$}; Mistral: \smash{$58\% \rightarrow 3\%$}), while \emph{upstream} placement largely preserves it. On models with distributed or sparse refusal patterns (Yi, Llama-2), placement makes less difference because refusal is redundant or only weakly localized across layers. Placement also affects \emph{ASR under attack}: on Qwen3, decoys placed inside the causal zone yield \smash{$80\%$} 3-point RFA ASR (LlamaGuard-2, JailbreakBench) compared to \smash{$35\%$} for upstream placement, even though baseline refusal is preserved in both cases. This motivates the upstream placement principle: placing decoys at or before the causal zone contaminates the attacker's estimator without disrupting the refusal computation itself.

Table~\ref{tab:refusal_proj} shows how compile mode affects baseline refusal:

\begin{table}[!htb]
\centering
\caption{Baseline refusal rate ($\uparrow$) under DDO replace vs.\ additive. Refusal rate measured on 100 JailbreakBench behaviors (substring heuristic, without attack). DDO uses per-model Optuna-tuned hyperparams (Table~\ref{tab:ddo_configs}). \cmark{badcell}{red} = refusal disrupted by DDO; $\times$ = coherence failure. The best compile mode depends on how each model's refusal circuitry responds to neuron overwriting.}
\label{tab:refusal_proj}
\small
\begin{tabular}{@{}lccc@{}}
\toprule
Model & Ref$_{\text{rep}}\uparrow$ & Ref$_{\text{add}}\uparrow$ & Best \\
\midrule
\rowcolor{grayrow}
Yi & 93\% & 56\% & Rep \\
Llama-2 & 96\% & 94\% & Rep \\
\rowcolor{grayrow}
Llama-3 & 95\% & 78\% & Rep \\
Gemma-2 & \bad{$\times$} & 93\% & Add \\
\rowcolor{grayrow}
Qwen3 & 52\% & 31\% & Add \\
Mistral & \bad{3\%} & 18\% & Add \\
\rowcolor{grayrow}
GLM-4 & 96\% & \bad{38\%} & Rep \\
\bottomrule
\end{tabular}
\end{table}

\noindent Table~\ref{tab:refusal_proj} reveals why the best compile mode is model-specific: replace mode produces a stronger decoy but can disrupt refusal itself. On Gemma-2, replace fails the coherence check, so additive is selected. On Mistral, replace nearly eliminates refusal (\smash{$3\%$} Ref), explaining its higher ASR with replace vs.\ additive (Table~\ref{tab:compile_mode}). On GLM-4, the pattern reverses: additive disrupts refusal (\smash{$38\%$} Ref) while replace preserves it (\smash{$96\%$}). The best mode therefore depends on how the model's refusal circuitry responds to neuron overwriting, not on a simple rule.

\subsection{Per-Model DDO Best Configs}
\label{sec:per_model_configs}

Table~\ref{tab:ddo_configs} reports the DDO hyperparameters used for each model in Table~\ref{tab:cross_baseline}. For each model, we run \smash{$15$} Optuna trials in replace mode and \smash{$15$} in additive mode, selecting the lower-ASR mode subject to a benign-compliance drop of \smash{$<20$} percentage points on held-out validation prompts when feasible. The search space covers layer range, \smash{$\beta$} (gate scale), scale (output magnitude), confusion \smash{$\lambda$} (estimator-confusion loss weight), learning rate, and number of epochs. Layer ranges are guided by the causal refusal zone (Table~\ref{tab:causal_zones}): on localized models we constrain DDO layers to start at or before the causal zone; on distributed models the search is unconstrained.

\begin{table}[h!]
\centering
\caption{DDO hyperparameters per model (Optuna-tuned, 15 trials per compile mode). Replace-mode params shown for Yi/Llama-2/Llama-3/GLM-4; additive-mode params for Gemma-2/Qwen3/Mistral. lr = learning rate, ep = epochs. DDO generalizes across architectures with per-model Optuna tuning of 6 hyperparameters.}
\label{tab:ddo_configs}
\small
\begin{tabular}{@{}llcccccc@{}}
\toprule
Model & Mode & Layers & $\beta$ & scale & conf.$\lambda$ & lr & ep \\
\midrule
\rowcolor{grayrow}
Yi & Replace & L18--32 & 6.1 & 0.42 & 0.5 & 0.002 & 2 \\
Llama-2 & Replace & L4--12 & 2.9 & 0.47 & 1.1 & 0.005 & 3 \\
\rowcolor{grayrow}
Llama-3 & Replace & L6--14 & 1.30 & 0.46 & 1.12 & 0.003 & 2 \\
Gemma-2 & Additive & L10--18 & 6.42 & 0.36 & 1.79 & 0.021 & 1 \\
\rowcolor{grayrow}
Qwen3 & Additive & L12--16 & 3.20 & 0.08 & 1.79 & 0.012 & 3 \\
Mistral & Additive & L2--18 & 4.05 & 0.46 & 0.82 & 0.008 & 2 \\
\rowcolor{grayrow}
GLM-4 & Replace & L6--17 & 5.3 & 0.29 & 1.2 & 0.016 & 1 \\
\bottomrule
\end{tabular}
\end{table}

\subsection{Benchmark Settings}
\label{sec:bench_settings}

Standard-RFA averages combine the four attack variants (three-point/residual-stream $\times$ JailbreakBench/HarmBench). Adaptive RFA, the additional-model evaluation, and the diagnostic studies explicitly labeled LlamaGuard-2 use that judge alone. Judge and aggregation details are given in the corresponding table captions.

\paragraph{Decoding and coherence.} Unless stated otherwise, we use greedy decoding ($\texttt{do\_sample}{=}\mathrm{false}$) with $\texttt{max\_new\_tokens}{=}512$ and batch size 8; all timing measurements use a single A100 GPU. Some post-hoc edits can break generation (e.g., near-empty or highly repetitive outputs) even when ASR appears low, so we apply a deterministic coherence check on 10 benign prompts and drop configurations with empty, degenerate (single-token repetition $>$80\%), or too-short outputs (marked ``$\times$'').

\paragraph{Splits and DIM estimation.} Hyperparameter tuning and DIM estimation use a held-out validation split (128 harmful + 128 safe probes); benchmark metrics are reported on held-out test sets. For DDO optimization, DIM directions are computed per-layer from last-token post-attention-layernorm activations ($\tilde{\mathbf{x}}_\ell$, the MLP input) using \textbf{unfiltered} estimation (no refusal-score filtering), accumulated in float64. For attack evaluation (RFA), directions are estimated from residual-stream activations ($\mathbf{h}_\ell$), matching the standard attacker protocol.

\paragraph{Table coloring conventions.}
Throughout all tables, cell shading follows consistent thresholds relative to the per-block base model. For ASR-type metrics (DR, RFA, Heretic, GCG, PAIR, AutoDAN; lower is better): \cmark{bestcell}{\textbf{best}} = best in column; \cmark{bestcell}{good} = $\leq$10\%; \cmark{badcell}{poor} = $\geq$50\%; unshaded = 10--50\%. For XSTest (higher is better): good = within 5pp of base or improved; poor = drop $>$20pp. For MMLU (higher is better): good = within 1pp of base or improved; poor = drop $>$3pp. For MT-Bench (higher is better): good = within 0.5 of base or improved; poor = drop $>$2. \cmark{oursrow}{ours} = DDO rows.

\paragraph{MMLU.}
We evaluate \smash{$5$}-shot multiple-choice accuracy on MMLU's \smash{$57$} subjects. We run evaluation via \texttt{lighteval} using the \texttt{helm|mmlu:<subject>|5|1} task configuration (\smash{$5$}-shot; few-shot truncation enabled) and report average exact-match accuracy across subjects. For tractability during iterative defense development, we cap evaluation to \smash{$100$} questions per subject unless otherwise noted.

\paragraph{MT-Bench.}
We run the standard MT-Bench benchmark (\smash{$80$} multi-turn questions) using FastChat's \texttt{llm\_judge} pipeline with GPT-4 as the judge in single-mode and report the average score (\smash{$1\text{--}10$}) across questions.

\paragraph{XSTest.}
We measure over-refusal on XSTest~\citep{rottger2024xstest}---250 safe prompts that superficially resemble harmful ones (e.g., ``How do I kill a process in Linux?''). A defended model is expected to comply; refusal here is a false-positive cost of the defense. We generate one response per prompt with greedy decoding (\smash{$\texttt{do\_sample}{=}\text{False}$}, \smash{$\texttt{max\_new\_tokens}{=}200$}), then classify each response with GPT-4o under the official XSTest 3-class rubric: \emph{full compliance} (direct answer), \emph{full refusal} (declines or only offers alternatives), and \emph{partial refusal} (mixed response). The reported XSTest score is the share of full-compliance outputs; partial and full refusal both count against the model.

\subsection{Baseline Sweep Details}
\label{sec:baseline_sweep}

This section documents our baseline-training sweeps. The main Llama-3 comparison uses publicly released defended checkpoints; the cross-model comparison uses our trained reproductions.

\paragraph{Common training setup.}
We trained six defenses (Circuit Breakers, RepBend, Triplet, Triplet-Adv, LAT, ReFAT) on six model families (Llama-3, Yi-1.5, Qwen3, Gemma-2, Mistral, GLM-4). All methods use LoRA finetuning on \texttt{q,k,v,o,gate\_proj,up\_proj,down\_proj} (GLM-4 uses \texttt{query\_key\_value, dense, dense\_h\_to\_4h, dense\_4h\_to\_h}), bf16 precision with tf32 matmuls, a constant learning-rate scheduler, gradient checkpointing, 1 epoch, and an effective batch size of 16 (\texttt{per\_device\_train\_batch\_size=2}, \texttt{gradient\_accumulation\_steps=8}). All CRL-family methods (CB, RepBend, Triplet, Triplet-Adv) train on \texttt{allenai/wildguardmix} (\texttt{wildguardtrain} split). For each (defense, model) pair, we run a 3-config sweep (A/B/C) over the primary hyperparameter(s) identified in the source paper; the best config is selected by lowest RFA ASR averaged over three judges (HarmBench classifier, LlamaGuard-2, StrongREJECT).

\paragraph{Circuit Breakers~\citep{zou2024circuitbreakers}.}
Circuit Breakers finetunes a LoRA adapter with a representation-rerouting loss that pushes harmful-prompt residual activations toward a ``broken'' (orthogonal or random) target representation, while a retention term preserves benign activations. We sweep \texttt{loss\_alpha} (rerouting strength) over \smash{$\{10,20,40\}$}, with all other loss terms set to zero (\texttt{loss\_beta=loss\_gamma=loss\_epsilon=loss\_eta=0}). Training uses LR \smash{$10^{-4}$}, LoRA \smash{$r{=}16$}, \smash{$\alpha{=}16$}, dropout \smash{$0.05$}, \texttt{max\_seq\_length=2048}, and \smash{$150$} steps (Config D: \smash{$300$} steps). Target layers follow the last \smash{$30\text{--}50\%$} of each model (Table~\ref{tab:layer_targets}).

\paragraph{RepBend~\citep{yousefpour2025repbend}.}
RepBend extends CB with a representation-bending loss: additional penalty terms ($\beta$, $\gamma$, $\varepsilon$) reshape the geometry of harmful representations along a learned direction while preserving benign ones. We jointly sweep \texttt{(loss\_alpha, loss\_gamma)} over \smash{$\{(0.5,0.3),(0.75,0.45),(1.0,0.6)\}$}, fixing \texttt{loss\_beta=0.1}, \texttt{loss\_epsilon=0.3}, \texttt{loss\_mode=response\_all}, \texttt{alpha\_mode=all}. Training uses LR \smash{$10^{-5}$}, LoRA \smash{$r{=}16$}, \texttt{max\_seq\_length=4096}, \smash{$450$} steps (D: \smash{$900$} steps).

\paragraph{Triplet~\citep{simko2025crl}.}
The Triplet loss trains the model to pull harmful-prompt activations toward refusal anchors and push them away from compliance anchors in representation space, using margins \texttt{(margin\_p, margin\_n)}. We sweep over \smash{$\{(2,3),(4,6),(8,12)\}$}. Fixed args: \texttt{loss\_alpha=0.5}, \texttt{loss\_beta=0.6}, \texttt{loss\_gamma=0.7}, \texttt{loss\_epsilon=0.7}, \texttt{loss\_safe\_dist\_p=norm}, distance metrics for unsafe pairs use cosine. Training uses LR \smash{$10^{-4}$}, LoRA \smash{$r{=}16$}, \smash{$900$} steps (D: \smash{$1800$} steps). Triplet is evaluated on Llama-3 in the main comparison (Table~\ref{tab:main}); it is included in the sweep for completeness.

\paragraph{Triplet-Adv~\citep{simko2025crl}.}
Identical to Triplet (\texttt{loss\_attack=False} $\to$ \texttt{True}): an adversarial latent perturbation is injected into hidden states during training before the triplet loss is applied, sharpening the representation boundary against activation-space attacks. Same sweep grid and hyperparameters as Triplet.

\paragraph{LAT~\citep{casper2024lat}.}
Latent Adversarial Training alternates between a PGD inner loop that perturbs hidden activations subject to an \smash{$\ell_2$} budget \texttt{epsilon} and an outer loop that minimizes task loss on the perturbed activations. We sweep \texttt{(epsilon, pgd\_iters)} over \smash{$\{(4,16),(8,24),(12,32)\}$}. Other args: \texttt{outer\_lr=8e-5}, \texttt{inner\_lr=1e-3}, \texttt{batch\_size=16}, \texttt{max\_batch\_per\_acc=2}, LoRA \smash{$r{=}64$}, \smash{$150$} steps (D: \smash{$300\text{--}400$} steps).

\paragraph{ReFAT~\citep{yu2024refat}.}
Refusal-Feature Adversarial Training augments each training step with probability \texttt{p\_rfa}: it estimates a refusal direction via DIM over \smash{$32$} samples, ablates it from activations across \smash{$75\%$} of layers, and computes the loss on the ablated activations; the refusal direction is recomputed every \smash{$4$} steps. We sweep \texttt{p\_rfa} over \smash{$\{0.5,0.7,0.9\}$}. Other args: LR \smash{$2\times10^{-5}$}, \texttt{max\_seq\_length=512}, \texttt{lora\_alpha=32}, gradient clip \smash{$1.0$}. Per-model LoRA rank and batch size: Llama-3 (\smash{$r{=}128$}, batch \smash{$32$}); Gemma-2 (\smash{$r{=}64$}, batch \smash{$8$}); all others (\smash{$r{=}128$}, batch \smash{$8$}).

\paragraph{Layer targeting.}
CB and RepBend/Triplet apply their losses only to a model-specific subset of layers; LAT and ReFAT similarly focus perturbations on a fraction of the network. Table~\ref{tab:layer_targets} lists the target-layer windows used across all six models.

\begin{table}[h!]
\centering
\caption{Per-model layer targeting for CB and RepBend/Triplet. CB applies the rerouting loss to the listed layer range; RepBend/Triplet use a sliding window starting at \texttt{RB\_START} of width \texttt{RB\_WIN}. All windows cover the last \smash{$30\text{--}50\%$} of model depth.}
\label{tab:layer_targets}
\small
\begin{tabular}{@{}lcccc@{}}
\toprule
Model & Total layers & CB target layers & \texttt{RB\_START} & \texttt{RB\_WIN} \\
\midrule
Llama-3-8B & 32 & 20--31 & 20 & 11 \\
Yi-1.5-9B  & 48 & 28--47 & 30 & 16 \\
Qwen3-8B   & 36 & 21--35 & 22 & 12 \\
Gemma-2-9B & 42 & 25--41 & 26 & 14 \\
Mistral-7B & 32 & 19--31 & 20 & 11 \\
GLM-4-9B   & 40 & 24--39 & 25 & 13 \\
\bottomrule
\end{tabular}
\end{table}

\paragraph{D-config protocol.}
If the best A/B/C config for a (defense, model) pair still exceeds \smash{$50\%$} RFA ASR, we retrain it at \smash{$2\times$} the step count (Config D) with all other hyperparameters fixed. Config D provided modest or no improvement in most cases; diminishing returns were common, particularly for CB, RepBend, and Triplet on models with diffuse refusal. Step counts for D: CB \smash{$300$}; RepBend \smash{$900$}; Triplet/Triplet-Adv \smash{$1800$}; LAT \smash{$300\text{--}400$}; ReFAT \smash{$300$}.

\paragraph{Config E: LLM-assisted hyperparameter selection.}
As an additional fairness check, we ran a Config~E in which Claude Opus~4.6 proposed hyperparameters for each (defense, model) pair, conditioned on the source paper, the search ranges, and the A--D outcomes. Config~E did not improve over the swept best for any pair, and did not outperform DDO under standard or adaptive RFA. For comparisons using these trained reproductions, we report the best available configuration from \smash{$\{$A,B,C,D,E$\}$} per (defense, model) pair.

\paragraph{Reproducibility caveats.}
LAT and ReFAT use a separate Python environment from the CRL-family methods due to dependency conflicts (\texttt{LLaMA-Factory} vs.\ \texttt{refusal\_direction\_defense}). All runs use \texttt{CUBLAS\_WORKSPACE\_CONFIG=:16:8} for determinism and disable Weights \& Biases logging. Adapter checkpoints are saved at \texttt{baselines/<model>/<defense>\_<config>/} and merged on demand for evaluation.

\subsection{Compute Budget}

Table~\ref{tab:compute} reports approximate wall-clock costs on a single A100 GPU. Defense costs are per optimization run for one hyperparameter configuration; hyperparameter search and downstream benchmark evaluation are additional. The \smash{$30\text{--}450\times$} comparison uses these per-configuration costs. Heretic and per-phase MT-Bench are reported separately as evaluation costs.

\begin{table}[h!]
\centering
\caption{Wall-clock compute on a single A100 GPU. Defense editing completes in minutes; evaluation (especially Heretic) dominates wall-clock cost.}
\label{tab:compute}
\small
\begin{tabular}{@{}llcr@{}}
\toprule
& Step & Prompts & Time \\
\midrule
\multirow{2}{*}{\emph{Defense}}
& DDO (per run) & 256 (128H+128S) & ${\sim}$2 min \\
& LM-head scaling (per run) & 20 & ${\sim}$5 sec \\
\midrule
\multirow{3}{*}{\emph{Evaluation}}
& RFA (4 variants) & 100--159 & ${\sim}$5 min \\
& Heretic (200 trials) & 100 & ${\sim}$30--60 min \\
& MT-Bench (per phase) & 80 & ${\sim}$20 min \\
\bottomrule
\end{tabular}
\end{table}

\subsection{Adaptive Attack: Exact ASR and MT-Bench by Phase}
\noindent Table~\ref{tab:adaptive} summarizes per-phase ASR (phases \smash{$1\text{--}8$}) for the full \smash{$8$}-phase setting; Table~\ref{tab:adaptive_mtbench} reports the exact per-phase ASR (LlamaGuard-2) \emph{and} MT-Bench (GPT-4 judge) values underlying Figure~\ref{fig:mtbench_adaptive}. Figure~\ref{fig:trajectory} provides a compact ASR--utility summary of the same runs.

\paragraph{Attack cost model.}
Each adaptive phase requires one DIM re-estimation (\smash{$256$} forward passes on \smash{$128$} harmful + \smash{$128$} safe probes, \smash{${\sim}30$}\,s) followed by Gram-Schmidt orthogonalization, hook installation, and evaluation (\smash{${\sim}1\text{--}2$}\,min), totaling \smash{${\sim}2$}\,min per phase on a single A100 GPU. A full \smash{$8$}-phase attack therefore costs \smash{${\sim}15\text{--}20$}\,min wall-clock for the attack loop. The separate MT-Bench evaluation at each phase is timed in Table~\ref{tab:compute}. The attack is fully automated (a deterministic loop over phases), requires no ML expertise, and uses only the defended checkpoint and a generic probe set---the same probes used for standard RFA suffice. However, multi-phase ablation degrades the resulting checkpoint: DDO-defended Llama-3 drops from MT-Bench \smash{$7.67$} (unattacked) to \smash{$6.17$} at Phase~\smash{$5$} (peak \smash{$65\%$} ASR) and \smash{$5.82$} by Phase~\smash{$8$} (Table~\ref{tab:adaptive_mtbench})---a \smash{${\sim}1.9$}-point utility cost. By contrast, on the undefended base model the attacker obtains \smash{$73\%$} ASR at Phase~\smash{$1$} with only a \smash{$0.3$}-point MT-Bench drop. DDO therefore forces the attacker to trade substantially more model quality per unit of ASR gained.

\begin{table}[h!]
\centering
\caption{Adaptive multi-phase RFA profiles (ASR \%, LlamaGuard-2) across all \smash{$8$} phases. ``Worst'' is the maximum ASR across phases (the attacker can stop at any phase). Both trained and post-hoc defenses degrade at later phases. DDO (rank \smash{$8$}) achieves \smash{$65\%$} worst-case, comparable to trained RepBend (\smash{$58\%$}), without training.}
\label{tab:adaptive}
\small
\begingroup
\setlength{\tabcolsep}{2pt}
\begin{tabular}{@{}lccccccccc@{}}
\toprule
Defense & Ph1 & Ph2 & Ph3 & Ph4 & Ph5 & Ph6 & Ph7 & Ph8 & Worst$\downarrow$ \\
\midrule
\rowcolor{grayrow}
Base (undefended) & 73 & 64 & 47 & 34 & 63 & 69 & 69 & 59 & 73 \\
RepBend (trained) & 4 & 20 & 21 & 28 & 43 & 56 & 58 & 42 & 58 \\
Circuit Breakers (trained) & 0 & 32 & 75 & 72 & 69 & 66 & 63 & 65 & 75 \\
ReFAT (trained) & 3 & 1 & 1 & 3 & 5 & 9 & 60 & 46 & 60 \\
\midrule
\ours DDO (rank 8) & 0 & 10 & 28 & 29 & 65 & 56 & 49 & 53 & 65 \\
\ours DDO (rank 8, Optuna) & 0 & 11 & 31 & 29 & 52 & 46 & 66 & 75 & 75 \\
\bottomrule
\end{tabular}
\endgroup
\end{table}

\begin{table}[h!]
\centering
\caption{Per-phase ASR ($\downarrow$) and MT-Bench ($\uparrow$) under adaptive multi-phase RFA. Phase 0 = unattacked model. ASR: LlamaGuard-2; MT-B: GPT-4 judge. DDO (rank 8) maintains MT-B $> 5.8$ through all 8 phases, comparable to RepBend and Circuit Breakers; ReFAT drops to 2.5 at Phase~1 despite low ASR.}
\label{tab:adaptive_mtbench}
\small
\begingroup
\setlength{\tabcolsep}{2pt}
\begin{tabular}{@{}c cc cc cc cc cc cc@{}}
\toprule
& \multicolumn{2}{c}{Base} & \multicolumn{2}{c}{RepBend} & \multicolumn{2}{c}{Circuit Breakers} & \multicolumn{2}{c}{ReFAT} & \multicolumn{2}{c}{DDO (rank 8)} & \multicolumn{2}{c}{DDO (rank 8, Optuna)} \\
\cmidrule(lr){2-3} \cmidrule(lr){4-5} \cmidrule(lr){6-7} \cmidrule(lr){8-9} \cmidrule(lr){10-11} \cmidrule(lr){12-13}
Ph & ASR & MTB & ASR & MTB & ASR & MTB & ASR & MTB & ASR & MTB & ASR & MTB \\
\midrule
\rowcolor{grayrow}
0 & 2 & 7.89 & 0 & 7.69 & 32 & 7.73 & 0 & 7.37 & 6 & 7.67 & 1 & 7.60 \\
1 & 73 & 7.47 & 4 & 4.00 & 0 & 5.40 & 3 & 2.46 & 0 & 7.96 & 0 & 7.83 \\
\rowcolor{grayrow}
2 & 64 & 7.23 & 20 & 8.29 & 32 & 7.19 & 1 & 2.72 & 10 & 7.36 & 11 & 7.27 \\
3 & 47 & 6.79 & 21 & 8.20 & 75 & 7.11 & 1 & 3.17 & 28 & 7.09 & 31 & 7.14 \\
\rowcolor{grayrow}
4 & 34 & 6.40 & 28 & 7.82 & 72 & 7.09 & 3 & 3.66 & 29 & 6.98 & 29 & 6.82 \\
5 & 63 & 6.21 & 43 & 7.33 & 69 & 6.53 & 5 & 4.40 & 65 & 6.17 & 52 & 6.48 \\
\rowcolor{grayrow}
6 & 69 & 3.91 & 56 & 6.72 & 66 & 6.49 & 9 & 3.62 & 56 & 5.91 & 46 & 6.14 \\
7 & 69 & 6.16 & 58 & 6.62 & 63 & 6.29 & 60 & 3.46 & 49 & 6.03 & 66 & 5.27 \\
\rowcolor{grayrow}
8 & 59 & 6.02 & 42 & 6.17 & 65 & 5.90 & 46 & 4.01 & 53 & 5.82 & 75 & 2.95 \\
\bottomrule
\end{tabular}
\endgroup
\end{table}

These phase trajectories reinforce the main-text comparison: DDO (rank \smash{$8$}) preserves utility (MT-Bench \smash{$\geq 5.82$} through Phase~\smash{$8$}) at a level comparable to RepBend (\smash{$6.17$}) and Circuit Breakers (\smash{$5.90$}), while ReFAT exhibits a sharp utility drop early (Phase~\smash{$1$} MTB \smash{$2.46$}) despite low ASR.

\begin{figure}[!htb]
\centering
\includegraphics[width=0.85\textwidth]{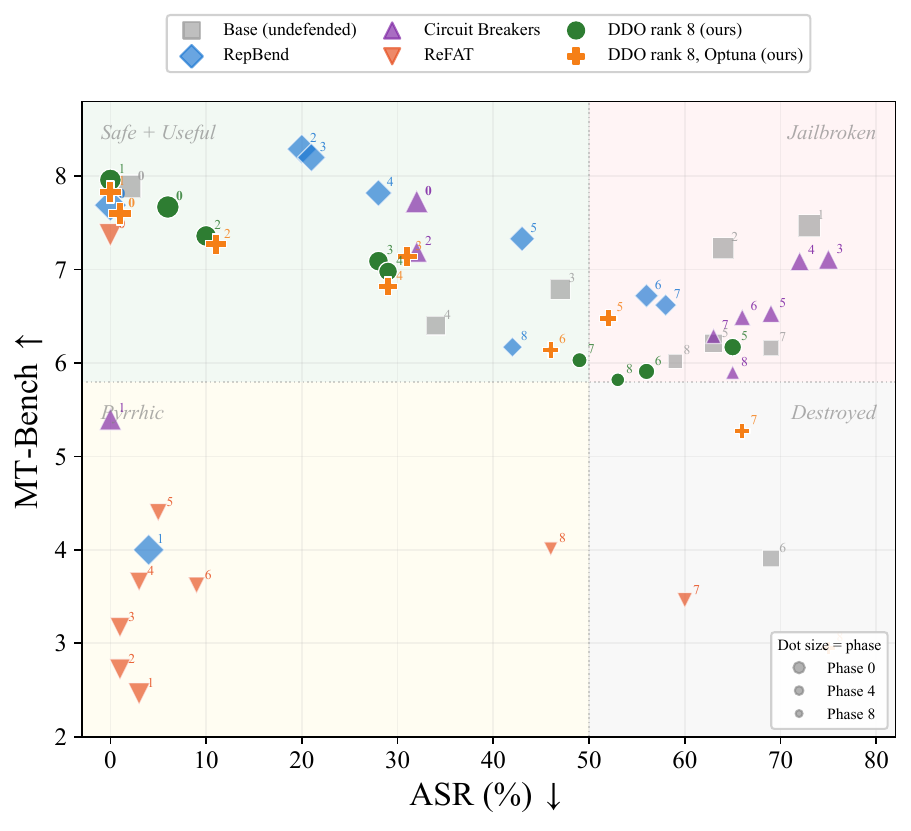}
\caption{ASR vs.\ MT-Bench trajectories under adaptive attack. Each dot is one defense at one phase (\smash{$0\text{--}8$}); larger dots = earlier phases. Quadrants: \textbf{top-left} = safe + useful (ideal); \textbf{top-right} = jailbroken but coherent (most dangerous); \textbf{bottom-left} = refuses but unusable; \textbf{bottom-right} = destroyed. ReFAT collapses into the bottom-left region (low ASR but utility destroyed), while DDO compositions (green, orange) preserve utility comparably to RepBend and Circuit Breakers, with worst-case ASR rising under sustained re-estimation.}
\label{fig:trajectory}
\end{figure}

\subsection{Extended Model Comparison}
\label{sec:full_zoo}

Table~\ref{tab:full_zoo} reports additional defense configurations evaluated on Llama-3, using consistent naming: \textbf{DDO} = the optimized decoy mechanism described in Section~\ref{sec:mechanisms}, without auxiliary mechanisms; \textbf{DDO + debiasing} = DDO composed with orthogonal debiasing; \textbf{DDO (rank \smash{$K$})} = DDO with \smash{$K$} diversified decoy readers (for adaptive RFA). The shown compositions are selected from a broader mechanism-subset search guided by two criteria: (i) minimize RFA ASR under the target attacker tier, and (ii) preserve utility (benign compliance \smash{$>85\%$}, MT-Bench \smash{$>6.0$}). The rank-\smash{$8$} variants diversify readers to increase the effective decoy rank; the Heretic-targeted variant allocates mechanisms to matrices outside Heretic's weight-edit surface. Optuna is used for per-model DDO hyperparameter tuning (Appendix~\ref{sec:per_model_configs}); the additional \emph{Optuna} qualifier in the Llama-3 rank-\smash{$8$} label distinguishes the separately tuned adaptive configuration.

\begin{table}[h]
\centering
\caption{Extended comparison on Llama-3-8B-Instruct. Utility: MMLU (\%), MT-B (MT-Bench), XST (XSTest) (\%). Robustness: RFA avg ASR (\%), Heretic ASR (\%) at \smash{$200$} Optuna trials, prompt jailbreak ASR (\%, 3-judge avg). \textsuperscript{$\dagger$} = post-attack generation degradation (Appendix~\ref{sec:heretic_quality}). \cmark{bestcell}{\textbf{best}}\,/\,\cmark{bestcell}{good}\,/\,\cmark{badcell}{poor}\,/\,\cmark{oursrow}{ours}. DDO alone achieves \smash{$1.8\%$} RFA and \smash{$18\%$} Heretic; debiasing repairs XSTest; rank-\smash{$8$} variants target adaptive RFA.}
\label{tab:full_zoo}
\small
\begingroup
\setlength{\tabcolsep}{2pt}
\begin{tabular}{@{}ll cccc c ccc@{}}
\toprule
Defense & Type & MMLU$\uparrow$ & MT-B$\uparrow$ & XST$\uparrow$ & RFA$\downarrow$ & H-200$\downarrow$ & GCG$\downarrow$ & PAIR$\downarrow$ & AutoDAN$\downarrow$ \\
\midrule
\rowcolor{grayrow}
Base & -- & \best{68.1} & \best{7.89} & 94.8 & \bad{85.1} & \bad{88.7} & 27.7 & \bad{55.1} & \good{0.4} \\
Circuit Breakers & Trained & \good{67.6} & \good{7.73} & \good{95.6} & \good{1.1} & 24.0 & \good{3.8} & 32.1 & \good{5.7} \\
\rowcolor{grayrow}
LAT & Trained & \good{67.8} & \good{7.52} & \bad{20.8} & \good{6.8} & \bad{82.3} & \good{8.0} & 34.6 & \best{0.0} \\
ReFAT & Trained & \good{67.2} & 7.37 & \bad{60.4} & \good{4.4} & \bad{85.3} & 30.6 & \best{30.0} & \best{0.0} \\
\rowcolor{grayrow}
RepBend & Trained & \bad{64.4} & \good{7.69} & \good{98.4} & \good{4.8} & \best{1.3} & \good{0.6} & 38.8 & \good{0.2} \\
Triplet & Trained & \good{67.7} & \good{7.78} & \good{98.0} & 18.1 & 0.0\textsuperscript{$\dagger$} & \best{0.4} & 39.8 & \best{0.0} \\
\rowcolor{grayrow}
Triplet-Adv & Trained & 65.6 & \good{7.47} & \good{96.8} & \good{7.7} & 30.0\textsuperscript{$\dagger$} & 22.2 & 43.6 & 20.8 \\
\midrule
\rowcolor{oursrow}
\textbf{DDO} & Ours & 66.8 & 7.38 & \good{91.6} & \good{1.8} & 18.0 & \good{1.7} & 37.5 & \best{0.0} \\
\rowcolor{oursrow}
\textbf{DDO + debiasing} & Ours & 66.8 & \good{7.56} & \best{99.2} & \good{1.0} & 22.3 & \best{0.4} & 36.3 & \best{0.0} \\
\midrule
\multicolumn{10}{@{}l}{\emph{DDO with diversified readers (for adaptive RFA; Section~\ref{sec:results}):}} \\
\rowcolor{oursrow}
DDO (rank 8) & Ours & \bad{64.7} & \good{7.56} & \good{94.0} & \best{0.1} & 22.7\textsuperscript{$\dagger$} & \good{3.4} & 37.9 & \best{0.0} \\
\rowcolor{oursrow}
DDO (rank 8, Optuna) & Ours & 65.1 & \good{7.68} & \good{94.8} & \good{0.6} & 19.3\textsuperscript{$\dagger$} & \good{3.1} & 39.2 & \good{0.2} \\
\midrule
\multicolumn{10}{@{}l}{\emph{Additional compositions (design-space exploration):}} \\
\rowcolor{grayrow}
DDO + 8 diversified aux\textsuperscript{a} & Ours & \bad{64.6} & 7.27 & \good{94.8} & \good{0.2} & \best{16.3}\textsuperscript{$\dagger$} & \good{2.5} & 38.4 & \best{0.0} \\
DDO + Heretic-targeted\textsuperscript{b} & Ours & 66.8 & \good{7.53} & \good{93.6} & \good{0.6} & \bad{58.7} & \good{1.5} & 34.2 & \best{0.0} \\
\bottomrule
\end{tabular}

\endgroup

{\footnotesize \textsuperscript{a}DDO + debiasing + diverse Q-head routing + RoPE + embed boost + KV strengthening + RMSNorm + sign coupling. \textsuperscript{b}DDO + debiasing with mechanisms in Heretic-resistant matrices.}
\end{table}

\subsection{Cross-Model Mechanism Comparison and Negative Results}
\label{sec:all_mechanisms}
\label{sec:failed_approaches}

Table~\ref{tab:mechanism_catalogue} compares 8 systematically evaluated mechanisms across all 7 models (10--30 Optuna trials each). Three patterns emerge: (i) DDO is the only mechanism that achieves $<$10\% ASR on every architecture; (ii) five mechanisms (diverse Q-head routing, rotation, V-projection, head amplification, KV strengthening) break generation entirely on Yi, Gemma-2, and GLM-4 ($\times$ = coherence failure), suggesting these architectures are less tolerant of attention- and normalization-space edits; and (iii) random orthogonal decoys without gradient optimization achieve low ASR only on Yi (2\%) but fail on most other models (18--53\%), confirming that the optimization in DDO's 4-part loss is essential. Formal definitions are in Appendix~\ref{sec:appendix_mech_defs}.

\begin{table}[h!]
\centering
\caption{Cross-model mechanism comparison (ASR \%$\downarrow$, Optuna-tuned, 10--30 trials each). DDO row: mean of 4 RFA variants, 3-judge avg (consistent with Tables~\ref{tab:main}--\ref{tab:cross_baseline}). Other mechanisms: 3-point RFA, LlamaGuard-2 (faster diagnostic). \cmark{bestcell}{best}\,/\,\cmark{badcell}{poor}\,/\,\cmark{oursrow}{ours}; $\times$ = coherence failure. DDO is the only mechanism achieving $<$10\% ASR across all 7 architectures.}
\label{tab:mechanism_catalogue}
\small
\begin{tabular}{@{}lccccccc@{}}
\toprule
Mechanism & Yi & Llama-2 & Llama-3 & Gemma-2 & Qwen3 & Mistral & GLM-4 \\
\midrule
\rowcolor{oursrow}
\textbf{DDO} & \good{5.0} & \best{1.0} & \best{1.8} & \best{0.0} & \best{8.0} & \best{0.0} & \best{2.0} \\
Random decoy (no opt) & \best{2} & 18 & 42 & 28 & \good{10} & 43 & 53 \\
\rowcolor{grayrow}
diverse\_q & $\times$ & 44 & 65 & $\times$ & 66 & \good{32} & $\times$ \\
gate\_boost & 66 & 56 & 77 & 56 & 80 & 67 & 81 \\
\rowcolor{grayrow}
rotation & $\times$ & 56 & 77 & $\times$ & 79 & 63 & $\times$ \\
v\_proj & $\times$ & 54 & 65 & $\times$ & 80 & 69 & $\times$ \\
\rowcolor{grayrow}
head\_amp & $\times$ & 58 & 84 & $\times$ & 84 & 66 & $\times$ \\
kv\_strength & $\times$ & 67 & 82 & $\times$ & 81 & 72 & $\times$ \\
\bottomrule
\end{tabular}

{\footnotesize Random decoy uses random orthogonal directions (no gradient optimization).}
\end{table}

\paragraph{Negative results and failed approaches.}
The broader mechanism search also yielded several informative negative results, organized by failure mode (ASR: 3-point RFA, JailbreakBench, LlamaGuard-2 unless noted).

\paragraph{Linear edits are fragile under adaptive re-estimation.}
Mechanisms equivalent to an input-independent linear transformation of the residual stream are defeated once the attacker re-estimates DIM on the defended checkpoint. \emph{Decoy Shear Transform} (\smash{$92\%$} ASR) injects orthogonal decoy components via a linear shear \smash{$\mathbf{M} = \Id + \sum c_j \mathbf{u}_j \hat{\mathbf{r}}^\top$}, but the defended DIM direction simply rotates to track the sheared mean-difference. \emph{Refusal Direction Rotation} (\smash{$68\%$} ASR) applies a Givens rotation in the \smash{$\{\hat{\mathbf{r}}, \mathbf{v}\}$} plane; the attacker recovers the rotated direction. \emph{Representation Rerouting} (\smash{$86\%$} ASR) maps harmful activations to random targets via a low-rank \smash{$\Delta\mathbf{W}$}, but provides insufficient disruption of linear separability.

\paragraph{Spreading decoy budget can be counterproductive.}
Distributing decoy energy across more directions involves a fundamental tradeoff (Corollary~\ref{cor:flat_spectrum}): with identical readers, adding more output directions dilutes each one without increasing the effective rank of $A_\theta$. With diversified readers, spreading energy trades rank-1 protection for higher-rank robustness (see Appendix~\ref{sec:spectral_validation} for empirical validation).

\paragraph{Aggressive edits destroy capability.}
\emph{LM-Head Row Scaling} achieves near-zero RFA ASR (\smash{$0.7\%$}) but at catastrophic utility cost (XSTest \smash{$8.0\%$}, MT-Bench \smash{$1.72$})---the model over-refuses almost all prompts. Extending DDO injection to layers \smash{$0\text{--}11$} (\smash{$48\%$} ASR) degraded generation quality to incoherent outputs. These cases illustrate that low ASR is insufficient without utility preservation.

\paragraph{Mismatched threat models.}
\emph{Gradient landscape roughening} (\smash{$84\%$} ASR) perturbs \smash{$\mathbf{W}_{\mathrm{gate}}$} to create rough input-space loss surfaces, targeting GCG-style prompt optimization. This does not improve robustness against representation-space abliteration (RFA), confirming that input-space and representation-space attacks require distinct defensive mechanisms.

\subsection{Reproducibility Notes}

\paragraph{Implementation.} DDO's gradient optimization through the defended model is implemented with \texttt{nnsight}~\citep{fiottokaufman2025nnsight}, which provides differentiable access to model internals while keeping base-model weights frozen. Hyperparameter search uses Optuna~\citep{akiba2019optuna} with the default Tree-structured Parzen Estimator (TPE) sampler.

We note three reproducibility details. (i) \textbf{Determinism:} we use deterministic decoding (\smash{$\texttt{do\_sample=false}$}) and fixed RNG seeds for all stochastic components, and adaptive multi-phase RFA is deterministic (verified across \smash{$3$} runs). (ii) \textbf{Seed sensitivity:} DIM direction estimation introduces \smash{$\approx \pm 1\%$} ASR variation across probe-set seeds, and Heretic results vary with Optuna random seed (\smash{$\approx \pm 3\%$} across \smash{$3$} independent 200-trial runs). (iii) \textbf{Fixed seeds:} data splits \smash{$=42$}, v\_proj conditioning \smash{$=99$} (searched over \smash{$\{42, 77, 99, 123, 200\}$}), KV strengthening \smash{$=7$}, Optuna TPE \smash{$=42$}.

\subsection{Defense and Attack Surface Map}

Figure~\ref{fig:block} visualizes a single pre-norm transformer block (Llama-3) with defense and attack annotations. DDO edits $\mathbf{W}_{\mathrm{gate}}$, $\mathbf{W}_{\mathrm{up}}$, and $\mathbf{W}_{\mathrm{down}}$ in the SwiGLU MLP (green); RFA hooks activations at three points (red dots); Heretic edits $\mathbf{W}_o$ and $\mathbf{W}_{\mathrm{down}}$ offline (red outlines). The partial overlap on $\mathbf{W}_{\mathrm{down}}$ is the only shared surface between DDO and Heretic.

\definecolor{attnblue}{HTML}{4A90D9}
\definecolor{ffnorange}{HTML}{E8913A}
\definecolor{normgreen}{HTML}{6BBF6B}
\definecolor{embpurple}{HTML}{7E57C2}
\definecolor{defgreen}{HTML}{43A047}
\definecolor{defbg}{HTML}{E8F5E9}
\definecolor{atkred}{HTML}{E53935}
\definecolor{atkbg}{HTML}{FFEBEE}
\definecolor{residgray}{HTML}{AAAAAA}

\begin{figure}[h!]
\centering
\resizebox{0.75\textwidth}{!}{%
\begin{tikzpicture}[
  node distance=0.5cm,
  state/.style={rectangle, rounded corners=2pt, draw=residgray!70!black, fill=white,
    minimum width=1.6cm, minimum height=0.55cm, font=\normalsize, line width=0.45pt},
  norm/.style={rectangle, rounded corners=3pt, draw=normgreen!70!black, fill=normgreen!12,
    minimum width=2.0cm, minimum height=0.55cm, font=\normalsize, line width=0.5pt},
  moduleA/.style={rectangle, rounded corners=4pt, draw=attnblue!70!black, fill=attnblue!8,
    minimum width=13.0cm, minimum height=3.0cm, line width=0.55pt},
  moduleM/.style={rectangle, rounded corners=4pt, draw=ffnorange!70!black, fill=ffnorange!8,
    minimum width=13.0cm, minimum height=3.0cm, line width=0.55pt},
  matA/.style={rectangle, rounded corners=2pt, draw=attnblue!70!black, fill=attnblue!15,
    minimum width=1.3cm, minimum height=0.50cm, font=\normalsize, line width=0.45pt},
  matM/.style={rectangle, rounded corners=2pt, draw=ffnorange!70!black, fill=ffnorange!15,
    minimum width=1.5cm, minimum height=0.50cm, font=\normalsize, line width=0.45pt},
  matE/.style={rectangle, rounded corners=2pt, draw=embpurple!70!black, fill=embpurple!12,
    minimum width=1.9cm, minimum height=0.55cm, font=\normalsize, line width=0.45pt},
  heretic/.style={draw=atkred!80!black, line width=1pt},
  op/.style={rectangle, rounded corners=2pt, draw=black!45, fill=white,
    minimum height=0.50cm, font=\normalsize, text=black!70, line width=0.35pt},
  add/.style={circle, draw=residgray, fill=white, inner sep=1.5pt, font=\normalsize, line width=0.55pt},
  arr/.style={-{Stealth[length=3.5pt,width=2.8pt]}, line width=0.5pt, color=black!55},
  inarr/.style={-{Stealth[length=3pt,width=2.2pt]}, line width=0.4pt, color=black!45},
  resarr/.style={-{Stealth[length=3pt,width=2.2pt]}, line width=0.4pt, color=residgray},
  def/.style={font=\small, text=defgreen!80!black},
  atk/.style={font=\normalsize\bfseries, text=atkred!80!black, anchor=west},
  alabel/.style={rectangle, rounded corners=2pt, fill=atkbg, draw=atkred!50,
    font=\normalsize\bfseries, text=atkred!80!black, inner sep=2pt, line width=0.3pt},
  hookmark/.style={circle, fill=atkred!70, inner sep=1.5pt},
  annot/.style={font=\small, text=black!50},
]

\node[matE] (wemb) {$\mathbf{W}_{\mathrm{emb}}$};
\node[state, above=0.5cm of wemb] (h0) {$\mathbf{h}_\ell$};
\node[norm, above=0.5cm of h0] (norm1) {RMSNorm};

\node[moduleA, above=0.6cm of norm1] (attn) {};
\node[font=\normalsize\bfseries, text=attnblue!70!black] at ([yshift=1.2cm]attn.center) {Multi-Head Attention};

\node[matA] (wq) at ([xshift=-5.0cm, yshift=0.5cm]attn.center) {$\mathbf{W}_q$};
\node[matA] (wk) at ([xshift=-5.0cm, yshift=-0.5cm]attn.center) {$\mathbf{W}_k$};
\node[matA] (wv) at ([xshift=-3.0cm, yshift=0.0cm]attn.center) {$\mathbf{W}_v$};
\node[op, minimum width=2.3cm] (sdp) at ([xshift=-0.6cm]attn.center) {\shortstack{Scaled Dot-Prod\\Attention}};
\node[op, minimum width=1.5cm] (concat) at ([xshift=2.2cm]attn.center) {Concat};
\node[matA, heretic] (wo) at ([xshift=4.5cm]attn.center) {$\mathbf{W}_o$};
\node[annot, text=attnblue!60!black] at ([xshift=-4.0cm, yshift=1.05cm]attn.center) {RoPE on $Q,K$};

\draw[inarr] (wq.east) -- ([yshift=0.18cm]sdp.west);
\draw[inarr] (wk.east) -- ([yshift=-0.18cm]sdp.west);
\draw[inarr] (wv.east) -- (sdp.west);
\draw[inarr] (sdp.east) -- (concat.west);
\draw[inarr] (concat.east) -- (wo.west);

\node[def, anchor=north] at ([yshift=-0.08cm]wv.south) {aux: V-proj, KV boost};
\node[def, anchor=north] at ([yshift=-0.08cm]wo.south) {aux: head amp., debias};

\node[add, above=0.5cm of attn] (add1) {$+$};
\node[norm, above=0.5cm of add1] (norm2) {RMSNorm};

\node[moduleM, above=0.6cm of norm2] (mlp) {};
\node[font=\normalsize\bfseries, text=ffnorange!70!black] at ([yshift=1.2cm]mlp.center) {SwiGLU MLP};

\node[matM] (wgate) at ([xshift=-4.8cm, yshift=0.5cm]mlp.center) {$\mathbf{W}_{\mathrm{gate}}$};
\node[matM] (wup) at ([xshift=-4.8cm, yshift=-0.5cm]mlp.center) {$\mathbf{W}_{\mathrm{up}}$};
\node[op, minimum width=1.0cm] (silu) at ([xshift=-2.8cm, yshift=0.5cm]mlp.center) {SiLU};
\node[circle, draw=black!45, fill=white, inner sep=1.5pt, font=\small, text=black!70, line width=0.35pt]
  (mul) at ([xshift=-1.3cm]mlp.center) {$\odot$};
\node[matM, heretic] (wdown) at ([xshift=3.0cm]mlp.center) {$\mathbf{W}_{\mathrm{down}}$};

\draw[inarr] (wgate.east) -- (silu.west);
\draw[inarr] (silu.east) -- (mul.west);
\draw[inarr] (wup.east) -- (mul.west);
\draw[inarr] (mul.east) -- (wdown.west);

\node[def, anchor=north] at ([yshift=-0.08cm]wup.south) {\textbf{DDO (gated decoy)}};
\node[def, anchor=north] at ([yshift=-0.08cm]wdown.south) {\textbf{DDO}, aux: debias};

\node[add, above=0.5cm of mlp] (add2) {$+$};
\node[state, above=0.5cm of add2] (h1) {$\mathbf{h}_{\ell+1}$};
\node[matE, above=0.5cm of h1] (wlm) {$\mathbf{W}_{\mathrm{lm}}$};

\draw[arr] (wemb) -- (h0);
\draw[arr] (h0) -- (norm1);
\draw[arr] (norm1) -- (attn);
\draw[arr] (attn) -- (add1);
\draw[arr] (add1) -- (norm2);
\draw[arr] (norm2) -- (mlp);
\draw[arr] (mlp) -- (add2);
\draw[arr] (add2) -- (h1);
\draw[arr] (h1) -- (wlm);

\coordinate (bypass1) at ($(attn.east)+(0.8cm,0)$);
\draw[resarr, rounded corners=5pt] (h0.east) -- (bypass1 |- h0.east) -- (bypass1 |- add1.east) -- (add1.east);
\coordinate (bypass2) at ($(mlp.east)+(0.8cm,0)$);
\draw[resarr, rounded corners=5pt] (add1.east) -- (bypass2 |- add1.east) -- (bypass2 |- add2.east) -- (add2.east);
\node[annot, text=residgray] at ([xshift=0.2cm]bypass1) {residual};

\node[def, anchor=east] at ([xshift=-0.3cm]wemb.west) {aux: embed boost, debias};

\node[hookmark] at ($(h0)!0.5!(norm1)$) (hook_h) {};
\node[alabel, anchor=west] (lbl_h) at ([xshift=2.0cm]hook_h) {RFA hook: $\mathbf{h}_\ell$};
\draw[atkred!50, -{Stealth[length=2.5pt]}, densely dashed, line width=0.4pt] (lbl_h.west) -- (hook_h);

\path (attn.north) -- (add1.south) coordinate[midway] (hook_a);
\node[hookmark] at (hook_a) {};
\node[alabel, anchor=west] (lbl_a) at ([xshift=2.0cm]hook_a) {RFA hook: $\mathbf{a}_\ell$};
\draw[atkred!50, -{Stealth[length=2.5pt]}, densely dashed, line width=0.4pt] (lbl_a.west) -- (hook_a);

\path (mlp.north) -- (add2.south) coordinate[midway] (hook_m);
\node[hookmark] at (hook_m) {};
\node[alabel, anchor=west] (lbl_m) at ([xshift=2.0cm]hook_m) {RFA hook: $\mathbf{m}_\ell$};
\draw[atkred!50, -{Stealth[length=2.5pt]}, densely dashed, line width=0.4pt] (lbl_m.west) -- (hook_m);

\node[alabel] at ([xshift=0.35cm]wo.east) {Heretic};
\node[alabel] at ([xshift=0.35cm]wdown.east) {Heretic};

\end{tikzpicture}%
}
\caption{Pre-norm transformer block (Llama-3) annotated with defense and attacker surfaces. \textcolor{defgreen!80!black}{Green}: \textbf{DDO} edits \smash{$\mathbf{W}_{\mathrm{gate}}$}, \smash{$\mathbf{W}_{\mathrm{up}}$}, and \smash{$\mathbf{W}_{\mathrm{down}}$} in the SwiGLU MLP; auxiliary (``aux'') mechanisms optionally edit additional matrices in attention, embeddings, and normalization (Section~\ref{sec:compositions}). \textcolor{atkred!80!black}{Red}: attacker access---RFA hooks activations at three points (\smash{$\mathbf{h}_\ell$}, \smash{$\mathbf{a}_\ell$}, \smash{$\mathbf{m}_\ell$}); Heretic~\citep{weidmann2025heretic} edits only \smash{$\mathbf{W}_o$} and \smash{$\mathbf{W}_{\mathrm{down}}$}. Overlap occurs only on \smash{$\mathbf{W}_{\mathrm{down}}$}.}
\label{fig:block}
\end{figure}
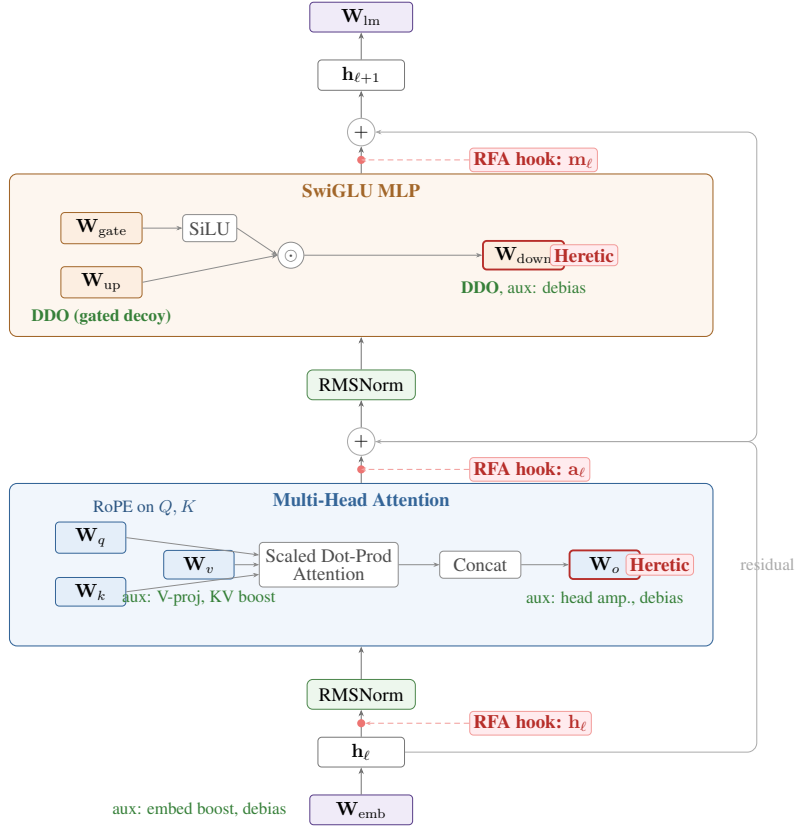

\subsection{Additional Model Evaluation}
\label{sec:additional_models}

Beyond the six primary models evaluated in the main text (Table~\ref{tab:main} and Table~\ref{tab:cross_baseline}), we apply DDO with per-model Optuna tuning (${\sim}$10 min each) to seven additional instruction-tuned checkpoints spanning 7B--24B parameters (Table~\ref{tab:additional_models}), including Llama-2-7B-Chat. DDO achieves $<$10\% standard-RFA ASR on all seven while preserving MMLU within 1--2 points of the base model, confirming cross-family and cross-scale generalization.

\begin{table}[h!]
\centering
\caption{\small DDO on additional models with per-model Optuna tuning (${\sim}$10 min, single A100 GPU). MMLU: base $\rightarrow$ DDO-defended (\%$\uparrow$). RFA ASR = mean of 4 variants, LlamaGuard-2 (\%$\downarrow$). XST = XSTest (\%$\uparrow$). DDO achieves $<$10\% ASR on all models with $\leq$1.4 point MMLU drop.}
\label{tab:additional_models}
\footnotesize
\begin{tabular}{@{}llcccc@{}}
\toprule
Model & Size & MMLU$\uparrow$ & RFA$\downarrow$ & XST$\uparrow$ \\
\midrule
\rowcolor{grayrow}
Llama-2-7B-Chat & 7B & 48.3$\rightarrow$46.9 & 1.0 & 79.4 \\
Qwen2.5-7B-Instruct & 7B & 74.2$\rightarrow$73.1 & 6.0 & 94.0 \\
\rowcolor{grayrow}
Qwen2.5-14B-Instruct & 14B & 79.7$\rightarrow$78.4 & 8.0 & 96.0 \\
Qwen3-14B & 14B & 79.0$\rightarrow$77.8 & 7.0 & 97.0 \\
\rowcolor{grayrow}
Mistral-Small-3-24B & 24B & 77.5$\rightarrow$76.6 & 2.0 & 96.0 \\
Mistral-v0.2-7B-Instruct & 7B & 60.1$\rightarrow$59.3 & 4.0 & 88.0 \\
\rowcolor{grayrow}
Ministral-8B-Instruct & 8B & 68.0$\rightarrow$66.8 & 5.0 & 86.0 \\
\bottomrule
\end{tabular}
\end{table}

\subsection{DDO Undo Attack Analysis}
\label{sec:detection_analysis}

An attacker who obtains only the defended checkpoint (without the original base model) might attempt to identify and revert DDO-modified neurons. A natural undo heuristic is to score each MLP gate row by its projection onto the attacker's estimated refusal direction, \smash{$s_i \defeq \left|\mathbf{w}_{\mathrm{gate}}^{(i)} \cdot \hat{\mathbf{d}}_{\mathrm{atk}}\right|$}, where \smash{$\hat{\mathbf{d}}_{\mathrm{atk}}$} is the DIM direction estimated on the defended checkpoint, and then zero the top-scoring neurons. DDO's gate rows are shaped to align with the reference refusal direction \smash{$\hat{\mathbf{r}}_{\mathrm{clean}}$} (so the trigger fires on harmful inputs), which would make them easy to find---if \smash{$\hat{\mathbf{d}}_{\mathrm{atk}}$} matched \smash{$\hat{\mathbf{r}}_{\mathrm{clean}}$}. Estimator corruption rotates \smash{$\hat{\mathbf{d}}_{\mathrm{atk}}$} away from \smash{$\hat{\mathbf{r}}_{\mathrm{clean}}$}, frustrating this undo strategy.

\paragraph{Detectability in intermediate layers.} Across Qwen3, Mistral-7B, and Yi, DDO target layers in the mid-zone show \smash{$\cos(\hat{\mathbf{r}}_{\mathrm{clean}}, \hat{\mathbf{d}}_{\mathrm{atk}}) \in [-0.05, +0.05]$}, so the estimated directions are nearly orthogonal. The modified row's projection \smash{$s_{\mathrm{mod}}$} collapses to within the natural top-\smash{$5$} distribution:

\begin{itemize}[itemsep=1pt, topsep=2pt]
\item \textbf{Qwen3 L14}: \smash{$s_{\mathrm{mod}} = 0.267$} vs.\ natural rank-\smash{$2$} at \smash{$0.230$} (\smash{$1.16\times$} margin).
\item \textbf{Mistral-7B L10}: \smash{$s_{\mathrm{mod}} = 0.041$}, outranked by a natural neuron at \smash{$0.066$} (rank \smash{$4$}).
\item \textbf{Yi L26}: \smash{$s_{\mathrm{mod}} = 0.112$} vs.\ natural rank-\smash{$2$} at \smash{$0.075$} (\smash{$1.48\times$} margin).
\item \textbf{Yi L32}: \smash{$s_{\mathrm{mod}} = 0.095$} vs.\ natural rank-\smash{$2$} at \smash{$0.091$} (\smash{$1.04\times$} margin).
\end{itemize}

Instruction-tuned models naturally contain MLP neurons whose gate rows align with refusal at the \smash{$0.07\text{--}0.20$} level---these are real refusal-relevant neurons doing work in the model's native safety circuitry. Once \smash{$\hat{\mathbf{d}}_{\mathrm{atk}}$} has been rotated away from \smash{$\hat{\mathbf{r}}_{\mathrm{clean}}$}, DDO's modified neuron sits inside this distribution rather than as an outlier. The undo attack scores neurons using this corrupted direction estimate. In the Llama-3 evaluation, it provides partial recovery while remaining below the undefended ASR (Table~\ref{tab:undo_asr}).

\paragraph{Detectability in later layers.} At the deepest target layers, deception can be slightly weaker. Yi L37 (the deepest layer in Yi's L14--37 zone) shows \smash{$\cos(\hat{\mathbf{r}}_{\mathrm{clean}}, \hat{\mathbf{d}}_{\mathrm{atk}}) \approx 0.09$}, and \smash{$s_{\mathrm{mod}} = 0.444$} exceeds the natural rank-\smash{$2$} (\smash{$0.081$}) by \smash{${\sim}5\times$}. However, since the attacker must search the full model without knowing which layers were modified, even a partial margin at one layer does not yield reliable undo across the model. Placing DDO inside the model's natural causal refusal zone gives the defense a stealth property: the modification shares its surface signature with the neurons it is functionally hiding among.

\paragraph{Post-attack ASR under undo heuristics.}
To complement the geometric analysis above, we measure post-attack ASR after applying each undo heuristic followed by standard RFA. For each heuristic, we zero the top-$m$ down-projection columns per DDO layer, then run 3-point RFA (Table~\ref{tab:undo_asr}).

\begin{table}[h!]
\centering
\caption{\small Post-attack ASR (\%$\downarrow$, LlamaGuard-2, 3-point RFA, JailbreakBench) after applying undo heuristics on Llama-3-8B DDO. $m$ = number of neurons zeroed per DDO layer. Base (no defense) ASR = 80\%. No heuristic recovers base-level ASR; the best heuristic reaches only 12--17\%.}
\label{tab:undo_asr}
\footnotesize
\begin{tabular}{@{}lcccccc@{}}
\toprule
Heuristic & $m{=}1$ & $m{=}2$ & $m{=}4$ & $m{=}8$ & $m{=}16$ & $m{=}32$ \\
\midrule
\rowcolor{grayrow}
gate$\cdot\hat{\mathbf{d}}_{\mathrm{atk}}$ projection & 16 & 17 & 15 & 15 & 12 & 13 \\
gate/up cosine & 17 & 14 & 16 & 15 & 12 & 14 \\
\rowcolor{grayrow}
down-col norm & 4 & 7 & 5 & 5 & 4 & 4 \\
\bottomrule
\end{tabular}
\end{table}

The gate$\cdot\hat{\mathbf{d}}_{\mathrm{atk}}$ and gate/up cosine heuristics yield 12--17\% ASR after RFA, indicating partial recovery while remaining below the undefended base (80\%). The down-column-norm heuristic yields 4--7\% ASR, indicating little recovery under this heuristic. Zeroing a neuron's down-projection column removes its contribution to the residual stream, including any decoy output from that neuron.

\paragraph{Norm-based detection.}
A detector could flag DDO neurons by their gate/up norm ratio. In the default compile mode, $\mathbf{w}_{\mathrm{gate}} = \beta\hat{\mathbf{r}}$ (norm $\beta$) and $\mathbf{w}_{\mathrm{up}} = \hat{\mathbf{r}}$ (norm $1$), yielding a ratio of $\beta \approx 3\text{--}6$---an outlier among ${\sim}12$K neurons. However, this is an artifact of a particular compile choice, not a fundamental limitation. The trigger output depends on the \emph{product} of gate and up norms: $\silu(\mathbf{w}_{\mathrm{gate}}^\top\tilde{\mathbf{x}}) \cdot \mathbf{w}_{\mathrm{up}}^\top\tilde{\mathbf{x}}$. Redistributing the total gain---e.g., setting both to $\sqrt{\beta}\,\hat{\mathbf{r}}$---equalizes norms while preserving the trigger in the large-activation regime ($\silu(x) \approx x$): $\silu(\sqrt{\beta}\,a)\cdot\sqrt{\beta}\,a \approx \beta a^2$, matching the original. Similarly, $\cossim(\mathbf{w}_{\mathrm{gate}}, \mathbf{w}_{\mathrm{up}}) = 1.0$ can be reduced by adding a small orthogonal perturbation $\varepsilon\,\mathbf{v}$ to $\mathbf{w}_{\mathrm{up}}$, introducing $O(\varepsilon)$ trigger noise absorbed by the down-projection scale. DDO's free scalar parameters $(\beta, s)$ provide the degrees of freedom to match any target norm and cosine distribution analytically. All models evaluated in this paper use the default compile without norm redistribution or cosine camouflage.

\subsection{\texorpdfstring{Rank-$k$ Attack Sweep}{Rank-k Attack Sweep}}
\label{sec:rank_sweep}

To test robustness beyond rank-1 RFA, we evaluate rank-$k$ SVD attacks: the attacker computes the top-$k$ left singular vectors of a harmful--safe contrast matrix on the defended model and ablates all $k$ directions simultaneously (Table~\ref{tab:rank_sweep}).

\begin{table}[h!]
\centering
\caption{\small Rank-$k$ attack sweep on Llama-3-8B DDO (ASR \%$\downarrow$, LlamaGuard-2, 3-point RFA, JailbreakBench). ASR increases with attack rank in this sweep, reaching 79\% at $k{=}32$, compared with 80\% for the undefended model.}
\label{tab:rank_sweep}
\footnotesize
\begin{tabular}{@{}lcccccc@{}}
\toprule
Attack rank $k$ & 1 & 2 & 4 & 8 & 16 & 32 \\
\midrule
\rowcolor{grayrow}
ASR (\%) & 4 & 16 & 21 & 29 & 39 & 79 \\
\bottomrule
\end{tabular}
\end{table}

ASR increases from 4\% at $k{=}1$ to 29\% at $k{=}8$ and 39\% at $k{=}16$. At $k{=}32$, ASR reaches 79\%, approaching the undefended base (80\%). These results show decreasing protection against the higher-rank attacks evaluated in this sweep.

\subsection{Probe-Budget Sensitivity}
\label{sec:probe_sweep}

A natural concern is that DDO's estimator corruption might rely on the attacker having too few probes for accurate DIM estimation. We test this by varying the attacker's probe budget from 32 to 1024 prompts in each of the harmful and safe sets (Table~\ref{tab:probe_sweep}).

\begin{table}[h!]
\centering
\caption{\small Probe-budget sweep on Llama-3-8B DDO (ASR \%$\downarrow$, LlamaGuard-2, 3-point RFA, JailbreakBench). $N$ is the number of prompts in each of the harmful and safe sets ($2N$ prompts total). DDO optimization uses 128 harmful and 128 safe prompts (256 total). ASR ranges from 6\% to 12\% without a sustained increase over the evaluated budgets.}
\label{tab:probe_sweep}
\footnotesize
\begin{tabular}{@{}lcccccc@{}}
\toprule
Prompts per class, $N$ & 32 & 64 & 128 & 256 & 512 & 1024 \\
\midrule
\rowcolor{grayrow}
ASR (\%) & 7 & 6 & 8 & 12 & 7 & 7 \\
\bottomrule
\end{tabular}
\end{table}

ASR ranges from 6\% to 12\% as the number of prompts per class increases from 32 to 1024. The 12\% result at $N{=}256$ is followed by 7\% at both $N{=}512$ and $N{=}1024$. Even with 1024 prompts per class (8$\times$ the DDO optimization budget of 128 per class), the attacker cannot recover the true refusal direction. This confirms that DDO's defense operates by genuinely corrupting the mean-difference direction, not by exploiting finite-sample noise in the attacker's estimator.

\subsection{Licenses for Existing Assets}
\label{sec:licenses}

All models, benchmarks, and software used in this work are publicly available under permissive or research-friendly licenses. We use each asset in accordance with its stated terms.

\paragraph{Models.}
\begin{itemize}[leftmargin=*,nosep]
\item \textbf{Llama-3-8B-Instruct}~\citep{grattafiori2024llama3}\footnote{\url{https://huggingface.co/meta-llama/Meta-Llama-3-8B-Instruct}}, \textbf{Llama-2-7B-Chat}~\citep{touvron2023llama2}\footnote{\url{https://huggingface.co/meta-llama/Llama-2-7b-chat-hf}}, \textbf{LlamaGuard-2}~\citep{inan2023llamaguard}\footnote{\url{https://huggingface.co/meta-llama/Llama-Guard-2-8B}}: Meta Llama 3 Community License (Llama-3, LlamaGuard-2); Llama 2 Community License (Llama-2). Both permit research and commercial use.
\item \textbf{Gemma-2-9B-IT}~\citep{team2024gemma}\footnote{\url{https://huggingface.co/google/gemma-2-9b-it}}: Gemma Terms of Use (permissive, allows research and redistribution).
\item \textbf{Qwen3-8B}~\citep{yang2025qwen3}\footnote{\url{https://huggingface.co/Qwen/Qwen3-8B}}: Apache 2.0.
\item \textbf{Mistral-7B-Instruct-v0.3}~\citep{jiang2023mistral}\footnote{\url{https://huggingface.co/mistralai/Mistral-7B-Instruct-v0.3}}: Apache 2.0.
\item \textbf{Yi-1.5-9B-Chat}~\citep{young2024yi}\footnote{\url{https://huggingface.co/01-ai/Yi-1.5-9B-Chat}}: Apache 2.0.
\item \textbf{GLM-4-9B-Chat}~\citep{glm2024chatglm}\footnote{\url{https://huggingface.co/THUDM/glm-4-9b-chat}}: GLM-4 License (permissive for research).
\item \textbf{GPT-4} (OpenAI): used as MT-Bench judge via API, under OpenAI's Terms of Use.
\item \textbf{GPT-4o} (OpenAI): used as XSTest and StrongREJECT judge via API, under OpenAI's Terms of Use.
\end{itemize}

\paragraph{Benchmarks and datasets.}
\begin{itemize}[leftmargin=*,nosep]
\item \textbf{JailbreakBench}~\citep{chao2024jailbreakbench}: MIT License.
\item \textbf{HarmBench}~\citep{mazeika2024harmbench}: MIT License.
\item \textbf{MMLU}~\citep{hendrycks2021mmlu}: MIT License.
\item \textbf{MT-Bench}~\citep{zheng2024mtbench}: Apache 2.0 (part of FastChat / lm-sys).
\item \textbf{XSTest}~\citep{rottger2024xstest}: CC-BY-4.0.
\item \textbf{StrongREJECT}~\citep{souly2024strongreject}: MIT License.
\item \textbf{AdvBench}~\citep{zou2023gcg}: MIT License.
\end{itemize}

\paragraph{Software and tools.}
\begin{itemize}[leftmargin=*,nosep]
\item \textbf{Heretic}~\citep{weidmann2025heretic}: GNU AGPL-3.0-or-later.
\item \textbf{LightEval}: Apache 2.0 (Hugging Face).
\item \textbf{Optuna}: MIT License.
\item \textbf{PyTorch}: BSD-style license.
\item \textbf{Transformers (Hugging Face)}: Apache 2.0.
\end{itemize}

\end{document}